\documentclass[review]{elsarticle}

\usepackage{lineno,hyperref}
\modulolinenumbers[5]

\journal{Journal of \LaTeX\ Templates}

\usepackage{color}
\usepackage{makecell}
\usepackage{multirow}

\usepackage{amsmath, amsthm, amssymb}
\usepackage{bbm}
\usepackage{algorithm}
\usepackage{algorithmic}
\newcommand{\bx}{\mathbf{x}}
\newcommand{\bh}{\mathbf{h}}
\newcommand{\bX}{\mathbf{X}}
\newcommand{\bB}{\mathbf{B}}
\newcommand{\bK}{\mathbf{K}}
\newcommand{\bD}{\mathbf{D}}
\newcommand{\bW}{\mathbf{W}}
\newcommand{\bV}{\mathbf{V}}
\newcommand{\bU}{\mathbf{U}}

\newcommand{\bI}{\mathbf{I}}

\newcommand{\bZ}{\mathbf{Z}}
\newcommand{\bS}{\mathbf{S}}
\newcommand{\bM}{\mathbf{M}}

\newcommand{\by}{\mathbf{y}}
\newcommand{\bY}{\mathbf{Y}}
\newcommand{\bu}{\mathbf{u}}

\newtheorem{prop}{Proposition}[section] 

\mathchardef\mhyphen="2D

\begin{document}

\begin{frontmatter}

\title{Spectral Clustering via Ensemble Deep Autoencoder Learning (SC-EDAE)}


\author[mymainaddress]{S\'everine Affeldt\corref{mycorrespondingauthor}}
\cortext[mycorrespondingauthor]{Corresponding author}
\ead{severine.affeldt@parisdescartes.fr}

\author[mymainaddress]{Lazhar Labiod}
\ead{lazhar.labiod@parisdescartes.fr}

\author[mymainaddress]{Mohamed Nadif}
\ead{mohamed.nadif@mi.parisdescartes.fr}

\address[mymainaddress]{University of Paris Descartes,\\ Mathematics and Computer Science,\\ 45 rue des Saints P\`eres,\\ 75006 Paris, France}

\begin{abstract}
Recently, a number of works have studied clustering strategies that combine classical clustering algorithms  and deep learning methods. These approaches follow
either a sequential way, where a deep representation is learned using a deep autoencoder before obtaining clusters with k-means, or a
simultaneous way, where deep representation and clusters are learned jointly by optimizing a single objective function. Both strategies improve clustering performance, however the robustness of these approaches is impeded by several deep autoencoder setting issues, among which the weights initialization, the width and number of layers or the number of epochs. To alleviate the impact of such hyperparameters setting on the clustering performance, we propose a new model which combines the spectral clustering and deep autoencoder strengths in an ensemble learning framework. Extensive experiments on various benchmark datasets demonstrate the potential and robustness of our approach compared to state-of-the-art deep clustering methods.
\end{abstract}

\begin{keyword}
spectral clustering\sep unsupervised ensemble learning\sep autoencoder
\end{keyword}

\end{frontmatter}


\section{Introduction}

Learning from large amount of data is a very challenging task. 
Several dimensionality reduction and clustering techniques that are well studied in the literature aim to learn a suitable and simplified data representation from original dataset; see for instance \cite{Yamamoto_14,AllabLN17,AllabLN18}.
While many approaches have been proposed to address the dimensionality reduction and clustering tasks, deep learning-based methods recently demonstrate promising results. Motivated by the keen interest in deep learning, many authors tackle the objective of data representation and partitioning using jointly the autoencoders~\cite{Hinton06} and clustering approaches. 
\subsection{Deep Autoencoder: challenges and issues}
Deep learning is a machine learning method that works with multi-level learning of data representations~\cite{bengio2009learning} where one passes from low level features to higher level features through the different layers. These deep architectures can automatically learn important features from images, sound or text data and have made significant progress in the field of computer vision. The autoencoder (AE) algorithm and its deep version (DAE), like the traditional methods of dimensionality reduction, has been a great success in recent years.

An autoencoder~\cite{Hinton06,Baldi12,BengioYAV13} is a neural network which is trained to replicate its input at its output. Training an autoencoder is unsupervised in the sense that no labeled data is needed. The training process is still based on the optimization of a cost function. Autoencoders can be used as tools to train deep neural networks~\cite{bengio2007greedy}.

For the purpose of dimensionality reduction, an autoencoder can learn a representation (or {\it encoding}) for a set of data. If linear activations are used, or only a single sigmoid hidden layer, then the optimal solution to an autoencoder is strongly related to Principal Component Analysis (PCA). With appropriate dimensionality and sparsity constraints, autoencoders can learn data projections that are more interesting than other basic techniques such as PCA which only allows linear transformation of data vectors.
By contrast, the autoencoders are non-linear by nature, and can learn more complex relations between visible and hidden units. Moreover, they can be stacked, which makes them even more powerful.

Recently, a number of works have studied clustering strategies that combine classical clustering algorithms and deep learning methods. These approaches follow either a sequential way, where a deep representation is learned using a deep autoencoder before obtaining clusters using a  clustering technique (e.g. \emph{k-means}) ~\cite{ShaoLDF15,TianGCCL14,WangHWW14,HuangHWW14,Leyli-AbadiLN17,YangCHWWZ16,BanijamaliG17,WangDF17,XieGF16}, or a simultaneous way, where deep representation and clusters are learned jointly by optimizing a single objective function~\cite{YangFSH17,Tian2017,Yang2016}. Both strategies improve clustering performance. However, when dealing with real-world data, existing clustering algorithms based on deep autoencoders suffer from different issues which impede their robustness and ease-to-use, such as,

\begin{itemize}
\setlength\itemsep{-0.25pt}
\item {\bf the weights initialization}, as mentioned in \cite{SeuretALI17}, the training of a Deep Neural Network (DNN) still suffers from two major drawbacks, among which the weights initialization. Indeed, initializing the weights with random values clearly adds randomness to the obtained results. The DNN pretraining~\cite{erhan2010does}, which is strongly related to the initialization issue, has been used in an increasing number of studies~\cite{YangFSH17,guo2017improved,xie2016unsupervised}. While pretraining helps to improve clustering performance, it is usually computationally intensive and thus raises supplementary training issues.
\item {\bf the architecture (or {\it structure})}, the architecture ({\it i.e.}, number of layers and their width) forces the network to seek a
different
representation of the data while preserving the 
important information. However, we observe that in almost all recent papers on deep clustering \cite{YangFSH17,Tian2017,Yang2016,BanijamaliG17,WangDF17,XieGF16,Ji2017}, a different structure is recommended by the authors for each studied dataset. 
In some studies, the DAE architecture can even lack of technical rationales. Most importantly, the clustering performance of the proposed methods usually strongly depends on a particular DAE structure.
\end{itemize}
\subsection{Our paper's contribution and structure}
To address the above mentioned challenging issues, we propose a 
Spectral Clustering via Ensemble Deep Autoencoder's algorithm ({\tt SC-EDAE}) 
which combines the advantages and strengths of spectral clustering, deep embedding models and ensemble paradigm. 
Ensemble learning has been considered in different machine learning context where it generally helps in improving results by combining several models. The ensemble approach allows a better predictive performance and a more robust clustering as compared to the results obtained with a single model. 
Following the ensemble paradigm, we first used several DAE with different hyperparameters settings to generate $m$ encodings. In a second step, each encoding is projected in a higher features space based on the {\it anchors} strategy \cite{Liu:2010,Chen11LandmarkSpectral} to construct $m$ graph affinity matrices. Finally, we apply spectral clustering on an ensemble graph affinity matrix to have the common space shared by all the $m$ encodings, before we run \emph{k-means} in this common subspace to produce the final clustering (see Fig.~\ref{fig:DAE-LSC_scheme} for a summary diagram).

The outline of the paper is as follows. In Section~\ref{sec:related_work} we present the related work. In Section~\ref{sec:preliminaries}, some notations and preliminaries are given. In Section~\ref{sec:SC-EDAE}, we present and discuss our approach in full details. In Section~\ref{sec:experiments}, the evaluations of the proposed method and comparisons with several related approaches available in the literature are presented. The conclusion of the paper is given in Section~\ref{sec:conclusion}.

\section{Related Work}
\label{sec:related_work}
Despite their success, most existing clustering methods are severely challenged by the data generated 
with modern applications, which are typically high-dimensional, noisy, heterogeneous and sparse. This has driven many researchers to investigate new clustering models to overcome these difficulties. One promising category of such models relies on data embedding.

Within this framework, classical dimensionality reduction approaches, e.g., Principal Component Analysis (PCA), have been widely considered for the embedding task. However, the linear nature of such techniques makes it challenging to infer faithful representations of real-world data, which typically lie on highly non-linear manifolds. This motivates the investigation of deep learning models (e.g., autoencoders, convolutional neural networks), which have been shown so far to be successful in extracting highly non-linear features from complex data, such as text, images or graphs~\citep{Hinton06,Baldi12,BengioYAV13}.

The deep autoencoders (DAE) have proven to be useful for dimensionality reduction~\cite{Hinton06} and image denoising. In particular, the autoencoders (AE) can non-linearly transform data into a latent space. When this latent space has lower dimension than the original one~\citep{Hinton06}, this can be viewed as a form of non-linear PCA. An autoencoder typically consists of an {\it encoder} stage, that can provide an encoding of the original data in lower dimension, and a {\it decoder} part, to define the data reconstruction cost. 
In clustering context, the general idea is to embed the 
data into a low dimensional latent space and then perform clustering in this new space. The goal of the embedding here is to learn new representations of the objects of interest (e.g., images) that encode only the most relevant information characterizing the original data, which would for example reduce noise and sparsity.

Several interesting works have recently combined embedding learning and clustering. The  proposed methods generally conduct both clustering and deep embedding in two different ways. First, some works proposed to combine deep embedding and clustering in a sequential way. In~\cite{TianGCCL14} the authors use a stacked autoencoder to learn a representation of the affinity graph, and then run \emph{k-means} on the learned representations to obtain the clusters. In~\cite{xie2016unsupervised}, it has been proposed to train a deep network by iteratively minimizing a Kullback-Leibler (KL) divergence between a centroid based probability distribution and an auxiliary target distribution. 

More recently, in~\cite{Guo:2017} the authors propose to incorporate an autoencoder into the Deep Embedded Clustering (DEC) framework~\citep{xie2016unsupervised}. Then, the proposed framework can jointly perform clustering and learn representative features with local structure preservation. A novel non-linear reconstruction method which adopt deep neural networks for representation based community detection has been proposed in~\citep{Yang2016}. The work presented in \citep{Ji2017} combines deep learning with subspace clustering such that the network is designed to directly learn the affinities matrix. Finally, a novel algorithm was introduced in~\citep{BanijamaliG17} that uses {\it landmarks} and deep autoencoders, 
to perform efficient spectral clustering.

Since the embedding process is not guaranteed to infer representations that are suitable for the clustering task, several authors recommend to perform both tasks jointly so as to let clustering govern feature extraction and vice-versa.
In \citep{Tian2017}, the authors propose a general framework, so-called {\it DeepCluster}, to integrate the traditional clustering methods into deep learning models and adopt Alternating Direction of Multiplier Method to optimize it. In \citep{YangFSH17}, a joint dimensionality reduction and \emph{k-means} clustering approach in which dimensionality reduction is accomplished via learning a deep neural network is proposed.

Beyond the joint and sequential ways to combine clustering and deep embedding, it appears that the connection between autoencoder and ensemble learning paradigm has not been explored yet. In this paper, we aim to fill the gap between ensemble deep autoencoders and spectral clustering in order to propose a robust approach that takes simultaneously advantage of several deep models with various hyperparameter settings.
In particular, we apply spectral clustering 
on an ensemble of {\it fused} encodings obtained from $m$ different deep autoencoders.
To our knowledge, the adoption of deep learning in an ensemble learning paradigm has not been adequately investigated yet. The goal of this work is to conduct investigations along this direction.

\section{Preliminaries}
\label{sec:preliminaries}
\subsection{Notation}
Throughout the paper, we use bold uppercase characters to denote matrices, bold lowercase characters to denote vectors. For any matrix $\mathbf{M}$, $\mathbf{m}_j$ denotes the $j$-th column vector of $\mathbf{M}$, $\mathbf{y}_i$ means the $i$-th row vector of $\mathbf{Y}$, $m_{ij}$ denotes the $(i,j)-$ element of $\mathbf{M}$ and $Tr[\mathbf{M}]$ is the trace of $\mathbf{M}$ whether $\mathbf{M}$ is a square matrix; $\mathbf{M}^{\top}$ denotes the transpose matrix of $\mathbf{M}$. We consider the Frobenius norm of a matrix $\mathbf{M} \in \mathbb{R}^{n \times d}$:  $||\mathbf{M}||^2=\sum_{i=1}^{n}\sum_{j=1}^{d}m_{ij}^2=Tr[\mathbf{M}^{\top}\mathbf{M}]$. 
Furthermore, let $\mathbf{I}$ be the identity matrix with appropriate size.

\subsection{Spectral clustering}
Spectral clustering is a popular clustering method that uses eigenvectors of a symmetric matrix derived from the distance between datapoints. Several algorithms have been proposed in the literature~\citep{verma2003comparison,von2007tutorial}, each using the eigenvectors in slightly different ways~\citep{shi2000normalized,ng2002spectral,meila2001learning}. The partition of the $n$ datapoints of $\bX ~\in \mathbb{R}^{n \times d}$ into 
$k$ 
disjoint clusters is based on an objective function that favors low similarity between clusters and high similarity within clusters. In its normalized version, the spectral clustering algorithm exploits the top $k$ eigenvectors of the normalized graph Laplacian $\mathbf{L}$ that are the relaxations of the indicator vectors which provide assignments of each datapoint to a cluster. In particular, it amounts to maximize the following relaxed normalized association,
\begin{equation}\label{eqSC}
  \max\limits_{\bB \in \mathbb{R}^{n\times k}} Tr(\bB^{\top} \bS \bB)\quad s.t. \quad \bB^{\top} \bB=\bI
\end{equation}
with $\bS = \bD^{-1/2} \bK \bD^{-1/2} ~\in \mathbb{R}^{n \times n}$ is the normalized similarity matrix where $\bK~\in \mathbb{R}^{n \times n}$ is the similarity matrix and $~D \in \mathbb{R}^{n \times n}$ is the diagonal matrix whose $(i,i)$-element of $\bX$ is the sum of $\bX$'s $i$-th row. The solution of (\ref{eqSC}) is to set the matrix $\bB~\in \mathbb{R}^{n \times k}$ equal to the $k$ eigenvectors corresponding to the largest $k$ eigenvalues of $\bS$. After renormalization of each row of $\bB$, a \emph{k-means} assigns each datapoint $\bx_i$ of $\bX$ to the cluster that the row $\mathbf{b}_i$ of $\bB$ is assigned to.

As opposed to several other clustering algorithms (e.g. \emph{k-means}), spectral clustering performs well on arbitrary shaped clusters. However, a limitation of this method is the difficulty to handle large-scale datasets due to the high complexity of the graph Laplacian construction and the eigendecomposition.

Recently, a scalable spectral clustering approach, referred to as {\it Landmark-based Spectral Clustering} (\emph{LSC})~\citep{chen2011large} or {\it AnchorGraph}~\citep{Liu:2010}, has been proposed. This approach allows to efficiently construct the graph Laplacian and compute the eigendecomposition. Specifically, each datapoint is represented by a linear combination of $p$ representative datapoints (or {\it landmarks}), with $p \ll n$. The obtained representation matrix $\Hat{\bZ} \in \mathbb{R}^{p\times n}$, for which the affinity is calculated between $n$ datapoints and the $p$ landmarks, is sparse which in turn ensures a more efficient eigendecomposition as compare to the above mentioned eigendecomposition of $\bS$ (Eq.~\ref{eqSC}).

\subsection{Deep autoencoders}

An autoencoder~\citep{hinton1994autoencoders} is a neural network that implements an unsupervised learning algorithm in which the parameters are learned in such a way that the output values tend to copy the input training sample. The internal hidden layer of an autoencoder can be used to represent the input in a lower dimensional space by capturing the most salient features.

Specifically, we can decompose an autoencoder in two parts, namely an {\it encoder}, $f_\theta$, followed by a {\it decoder}, $g_\psi$. The first part allows the computation of a feature vector $\mathbf{y}_i = f_\theta(\mathbf{x}_i)$ for each input training sample, thus providing the {\it encoding} $\mathbf{Y}$ of the input dataset. The {\it decoder} part aims at transforming back the encoding into its original representation, $\hat{\mathbf{x}}_i = g_\psi(\mathbf{y}_i)$. 

The sets of parameters for the encoder $f_\theta$ and the decoder $g_\psi$ are learned simultaneously during the reconstruction task while minimizing the {\it  loss}, referred to as $\mathcal{J}$, where $\cal{L}$ is a cost function for measuring the divergence between the input training sample and the reconstructed data,
\begin{equation}\label{eqDAEloss}
  \mathcal{J}_{AE}(\theta, \psi)=\sum\limits_{i=1}^n {\cal{L}}(\bx_i, g_\psi(f_\theta(\bx_i))).
\end{equation}
The encoder and decoder parts can have several shallow layers, yielding a deep autoencoder (DAE) that enables to learn higher order features. The network architecture of these two parts usually mirrors each other.

It is remarkable that PCA can be interpreted as a linear AE with a single layer \citep{Hinton06}. In particular,  PCA can be seen as a linear autoencoder with $\bW \in \mathbb{R}^{d \times k}$ where $k \le d$.  Taking $f_\theta(\bX)=\bX\bW$ and $g_\psi \circ f_\theta(\bX)=\bX \bW\bW^{\top}$ we find the objective function $|| \bX - \bX \bW\bW^{\top}||^{2}$ optimized by PCA.

\section{ Spectral Clustering via Ensemble DAE}
\label{sec:SC-EDAE}

\subsection{Problem formulation}
Given an $n \times d$ data matrix $\bX$, the goal is to first obtain a set of $m$ encodings $\{\bY_\ell\}_{\ell \in [1,m]}$ using $m$ DAE trained with different hyperparameters settings. In a second step, we construct a graph matrix $\bS_\ell$ associated to each embedding $\bY_\ell$, and then {\it fuse} the $m$ graph matrices in an ensemble graph matrix $\bS$ which contains information provided by the $m$ embeddings. Finally, to 
benefit from the common subspace shared by the $m$ deep embeddings, spectral clustering is applied to $\bS$. The challenges of the problem  are threefold,

\begin{enumerate}
    \item generate $m$ deep embeddings,
    \item integrate the clustering in an ensemble learning framework,
    \item solve the clustering task in a highly efficient way.
\end{enumerate}

Each of the above mentioned issues is discussed in the separate subsections~\ref{ssec:deep_embed_generation}, \ref{ssec:graph_mat_construction} and \ref{ssec:ens_affin_mat} respectively. Most importantly, the {\tt SC-EDAE} approach is provided with an ensemble optimization which is detailed in subsection~\ref{ssec:optimization}.

\subsection{Deep embeddings generation}
\label{ssec:deep_embed_generation}
The cost function of an autoencoder, with an encoder $f_\theta$ and a decoder $g_\psi$, measures the error between the input $\bx \in \mathbb{R}^{d\times 1}$ and its reconstruction at the output $\hat{\bx} \in \mathbb{R}^{d\times 1}$. 
The encoder $f_\theta$ and decoder $g_\psi$ can have multiple layers of different widths. To generate $m$ deep representations or encodings $\{\bY_\ell\}_{\ell \in [1,m]}$, the DAE is trained with different hyperparameter settings (e.g., initialization, layer widths) by optimizing the following cost function. 
\begin{equation}\label{eqDEA}
  || \bX - g_{\psi_\ell} \circ f_{\theta_\ell}(\bX)||^{2}  
\end{equation}
where $g_{\psi_\ell}$ and $f_{\theta_\ell}$ are learned with the hyperparameter setting $\ell$, and $\bY_\ell = f_{\theta_\ell}(\bX)$ (Fig.~\ref{fig:DAE-LSC_scheme}, $(a)$).

\subsection{Graph matrix construction}
\label{ssec:graph_mat_construction}
To construct the 
graph matrix $\bS_\ell$, we use an idea similar to that of {\it Landmark Spectral Clustering}~\cite{Chen11LandmarkSpectral} and  the {\it Anchor-Graphs}~\cite{Liu:2010}, where a smaller and sparser
representation matrix $\bZ_\ell \in \mathbb{R}^{n \times p}$ that approximates a full $n \times n$ affinity matrix is built between the landmarks $\{\bu_j^\ell\}_{j \in [1,p]}$ and the encoded points $\{\by_i^\ell\}_{i \in [1,n]}$ (Fig.~\ref{fig:DAE-LSC_scheme}, $(a)$). Specifically, a 
set of $p$ points ($p \ll n$) are obtained through a $k$-means clustering on the embedding matrix $\bY_\ell$. 
These points are the landmarks
which approximate the neighborhood structure. Then a non-linear mapping from data to landmark is computed as follows,
\begin{equation}\label{eqanchor}
  z_{ij}^\ell = \Phi(\by_i^\ell)=\frac{\mathcal{K}(\by_i^\ell,\bu_j^\ell)}{\sum_{j' \in N_{(i)} }\mathcal{K}(\by_i^\ell,\bu_{j'}^\ell)} 
\end{equation}
where $N_{(i)}$ indicates the $r$ ($r<p$) nearest landmarks around $\by_i^\ell$. As proposed in~\cite{Chen11LandmarkSpectral}, we set $z_{ij}^\ell$ to zero when the landmark $\bu_j^\ell$ is not among the nearest neighbor of $\by_i^\ell$, leading to a sparse affinity matrix $\bZ_\ell$. The function $\mathcal{K}(.)$ is used to measure the similarity between data $\by_i^\ell$  and anchor $\bu_j^\ell$ with $L_2 $ distance in Gaussian kernel space $\mathcal{K}(\bx_i, \bx_j) = \exp(-||\bx_i -\bx_j||^2/2\sigma^2)$, and $\sigma$ is the bandwidth parameter. The normalized matrix $\hat{\bZ}_\ell \in \mathbb{R}^{n \times p}$ is then utilized to obtain a low-rank 
graph matrix,  
\begin{eqnarray*}
\bS_\ell \in \mathbb{R}^{n\times n}, \quad
\mathbf{S}_\ell = {\bZ}_\ell \Sigma^{-1}{\bZ}_\ell^{\top} \mbox{ where } \Sigma = diag({\bZ}_\ell^{\top}\mathbbm{1}).
\end{eqnarray*}
As the $\Sigma^{-1}$  normalizes the constructed matrix, $\mathbf{S}_\ell$ is bi-stochastic, {\it i.e.} the summation of each column and row equal to one, and the graph Laplacian becomes, 
\begin{equation}
  \bS_\ell = \hat{\bZ}_\ell\hat{\bZ}_\ell^{\top} \quad \mbox{where}~~ \hat{\bZ}_\ell=\bZ_\ell \Sigma^{-1/2}.
\end{equation}

\subsection{Ensemble of affinity matrices}
\label{ssec:ens_affin_mat}
Given a set of $m$ 
encodings $\{\bY_\ell\}_{\ell \in [1, m]}$ obtained using  $m$ DAE trained with different hyperparameters setting $\ell$, the goal is to merge the $m$ 
graph similarity matrices $\bS_\ell$ in an ensemble similarity matrix which contains information provided by the $m$ embeddings. To  aggregate the different similarity matrices, we use an Ensemble Clustering idea analogous to that proposed in~\cite{Strehl:2003,vega2011survey}
where a {\it co-association} matrix is first built as the summation of all basic similarity matrices, and where each basic partition matrix can be represented as a block diagonal matrix. Thus, the {\tt SC-EDAE} ensemble affinity matrix is built as the summation of the $m$ basics similarity matrices using the following formula,
\begin{equation}\label{eqSbar}
\bar{\bS}=\frac{1}{m}\sum_{\ell=1}^{m}\bS_\ell.
\end{equation}
Note that the obtained matrix $\bar{\bS}$ is bi-stochastic, as ${\bS}_\ell$ (Eq.~\ref{eqSbar}). For many natural problems, $\bar{\bS}$ is approximately block stochastic matrix, and hence the first $k$ eigenvectors of $\bar{\bS}$ are approximately piecewise constant over the $k$ almost invariant rows subsets~\cite{Maila-2001}. 

In the sequel, we aim to compute, at lower cost, $\bB$ that is shared by the $m$ graph matrices $\bS_\ell$, and obtained by optimizing the following trace maximization problem
\begin{equation}\label{Sbar}
  \max_{\bB} Tr(\bB^{\top} \bar{\bS} \bB) \quad s.t. \quad \bB^{\top} \bB=\bI.
\end{equation}

\subsection{Proposed optimization and algorithm}
\label{ssec:optimization}

The solution of Eq.~\ref{Sbar} is to set the matrix $\bB$ equal to the $k$ eigenvectors corresponding to the largest $k$ eigenvalues of $\bar{\bS}$. 
However, as the computation of the eigendecomposition of  $\bar{\bS}$ of size $(n \times n)$ is $O(n^3)$, relying on proposition \ref{propo1}, we propose instead to compute the $k$ left singular vectors of the concatenated matrix,
\begin{equation}\label{eqzbar}
\bar{\bZ}=\frac{1}{\sqrt{m}}[ \hat{\bZ}_1| \ldots| \hat{\bZ}_j| \ldots |\hat{\bZ}_m].
\end{equation}
Using the sparse matrix $\bar{\bZ}$ $\in \mathbb{R}^{n \times \sum_{j=1}^{m} \ell_{j}}$ with $\sum_{j=1}^{m} \ell_{j}\ll n$, instead of $\bar{\bS}$, which has a larger dimension, naturally induces an improvement in the computational cost of $\bB$ (Fig.~\ref{fig:DAE-LSC_scheme}, $(b)$).

\begin{prop}\label{propo1}
Given a set of $m$ similarity matrices $\bS_\ell$, such that each matrix $\bS_\ell$ can be expressed as $\bZ_\ell\bZ_\ell^{\top}$. Let $\bar{\bZ} \in \mathbb{R}^{n \times \sum_{j=1}^{m} \ell_{j}}$, where 
$\sum_{j=1}^{m} \ell_{j} \ll n$, denoted as $\frac{1}{\sqrt{m}}[ \bZ_1| \ldots | \bZ_j| \ldots |\bZ_m]$, be the concatenation of the $\bZ_\ell$'s,   $\ell=1,\ldots,m$. We first have,
\begin{equation}\label{equiveq}
\max_{\bB^{\top}\bB=\mathbf{I}}Tr(\bB^{\top} \bar{\bS} \bB) \Leftrightarrow \min_{\bB^{\top}\bB=\mathbf{I},\bM}||\bar{\bZ}-\bB \bM^{\top}||_F^2.
\end{equation}
Then, given ${\tt SVD}(\bar{\bZ})$, $\bar{\bZ}=\bU\Sigma\bV^{\top}$ and the optimal solution $\bB^{*}$ is equal to $\bU$.

\end{prop}

\begin{proof}
From the second term of Eq.~\ref{equiveq}, one can easily show that $\bM^{*}=\bar{\bZ}^{\top}\bB$. Plugging now the expression of $\bM^{*}$ in Eq.~\ref{equiveq}, the following equivalences hold
\begin{eqnarray*}
\min_{\bB^{\top}\bB=\mathbf{I},\bM}||\bar{\bZ}-\bB \bM^{\top}||_F^2 &\Leftrightarrow& \min_{\bB^{\top}\bB=\mathbf{I}}||\bar{\bZ}-\bB\bB^{\top} \bar{\bZ}||_F^2\\
&\Leftrightarrow& \max_{\bB^{\top}\bB=\mathbf{I}}Tr(\bB^{\top} \bar{\bZ}\bar{\bZ}^{\top} \bB)\\
&\Leftrightarrow& \max_{\bB^{\top}\bB=\mathbf{I}}Tr(\bB^{\top} \bar{\bS}\bB).
\end{eqnarray*}
On the other hand, ${\tt SVD} (\bar{\bZ})$ leads to $\bar{\bZ}=\bU \Sigma \bV^{\top}$  (with $\bU^{\top}\bU=\bI$, $\bV^{\top}\bV=\bI$)  and therefore to the eigendecomposition of $\bar{\bS}$ as follows: 
\begin{eqnarray*}
\bar{\bS}=\bar{\bZ}\bar{\bZ}^{\top}&=&(\bU \Sigma \bV^{\top})(\bU \Sigma \bV^{\top})^\top\nonumber\\
&=&\bU \Sigma (\bV^{\top}\bV) \Sigma \bU^{\top}\nonumber\\
&=&\bU \Sigma^2 \bU^{\top}.
\end{eqnarray*}
Thereby the left singular vectors of $\bar{\bZ}$ are the same as the eigenvectors of $\bar{\bS}$.
\end{proof}

The steps of our {\tt SC-EDAE} algorithm 
are summarized in Algorithm \ref{Alg_sc-edae} and illustrated by Figure~\ref{fig:DAE-LSC_scheme}. 
The {\tt SC-EDAE} approach proposes a unique way to combine DAE encodings with clustering. It also directly benefits from the 
low complexity of the \emph{anchors} strategy for both the graph affinity matrix construction and the eigendecomposition.

Specifically, the computational cost for the construction of each $\bZ_\ell$ affinity matrix amounts to $\mathcal{O}(np_{\ell}e(t+1))$ (Alg.~\ref{Alg_sc-edae}, step (b)) , where $n$ is the number of datapoints, $p_\ell$ is the number of landmarks for the $\ell^{th}$ DAE ($p_\ell \ll n$), $e$ is the size of the DAE encoding $\bY_\ell$ ($e \ll n$) and $t$ is the number of iterations for the {\it k-means} that is used to select the landmarks. The computation of the $\bZ_\ell$ matrices can be easily parallelized over multiple cores, leading to an efficient computation of the ensemble affinity matrix $\bar{\bZ}$.
Furthermore, the eigendecomposition of the sparse ensemble affinity matrix $\bar{\bZ}$, which leads to the $\bB$ embeddings (Alg.~\ref{Alg_sc-edae}, step (c)), induces a computational complexity of $\mathcal{O}(p'^3+p'2n)$, where $p'$ is the sum of all landmarks numbers for the concatenated $\bZ_\ell$ matrices, {\it i.e.} $p'=\sum_{j=1}^{m} \ell_{j}\ll n$. Finally, we need additional $\mathcal{O}(nctk)$ for the last {\it k-means} on $\bB \in \mathbbm{R}^{n \times k}$ (Alg.~\ref{Alg_sc-edae}, output), where $c$ is the number of centro\"ids, usually equal to $k$ the number of eigenvectors, leading to $\mathcal{O}(ntk^2)$.
%
%
\begin{algorithm}[ht!]
\caption{: {\tt SC-EDAE} algorithm}
\label{Alg_sc-edae}
\begin{algorithmic}
\STATE {\bfseries Input:}  data matrix $ \mathbf{X}$;
\STATE {\bfseries Initialize:} $m$ DAE with different hyperparameters setting; \\
%
\STATE {\bfseries Do:} \\
\STATE
 (a) Generate $m$ deep embedding $\{\bY_\ell\}_{l \in [1, m]}$ \hfill {\footnotesize{(\it Eq.~\ref{eqDEA})}}\\
 (b) Construct the ensemble sparse affinity matrix  $ \mathbf{\bar{Z}}$ $\in \mathbb{R}^{n \times \sum_{j=1}^{m} \ell_{j}}$ \hfill {\footnotesize{\it (Eq.~\ref{eqanchor}, \ref{eqzbar})}}\\
 (c) Compute $ \mathbf{B^* \in \mathbbm{R}^{n \times k}}$ by  performing {\it sparse} {\tt SVD} on  $\mathbf{\bar{Z}}$ \hfill {\footnotesize{\it (Eq.~\ref{equiveq})}} \\
\STATE {\bfseries Output:}   Run $k$-means on $\mathbf{B^*}$ to get the final clustering
\end{algorithmic}
\end{algorithm}
\begin{figure*}[ht!]
\includegraphics[scale = 0.52]{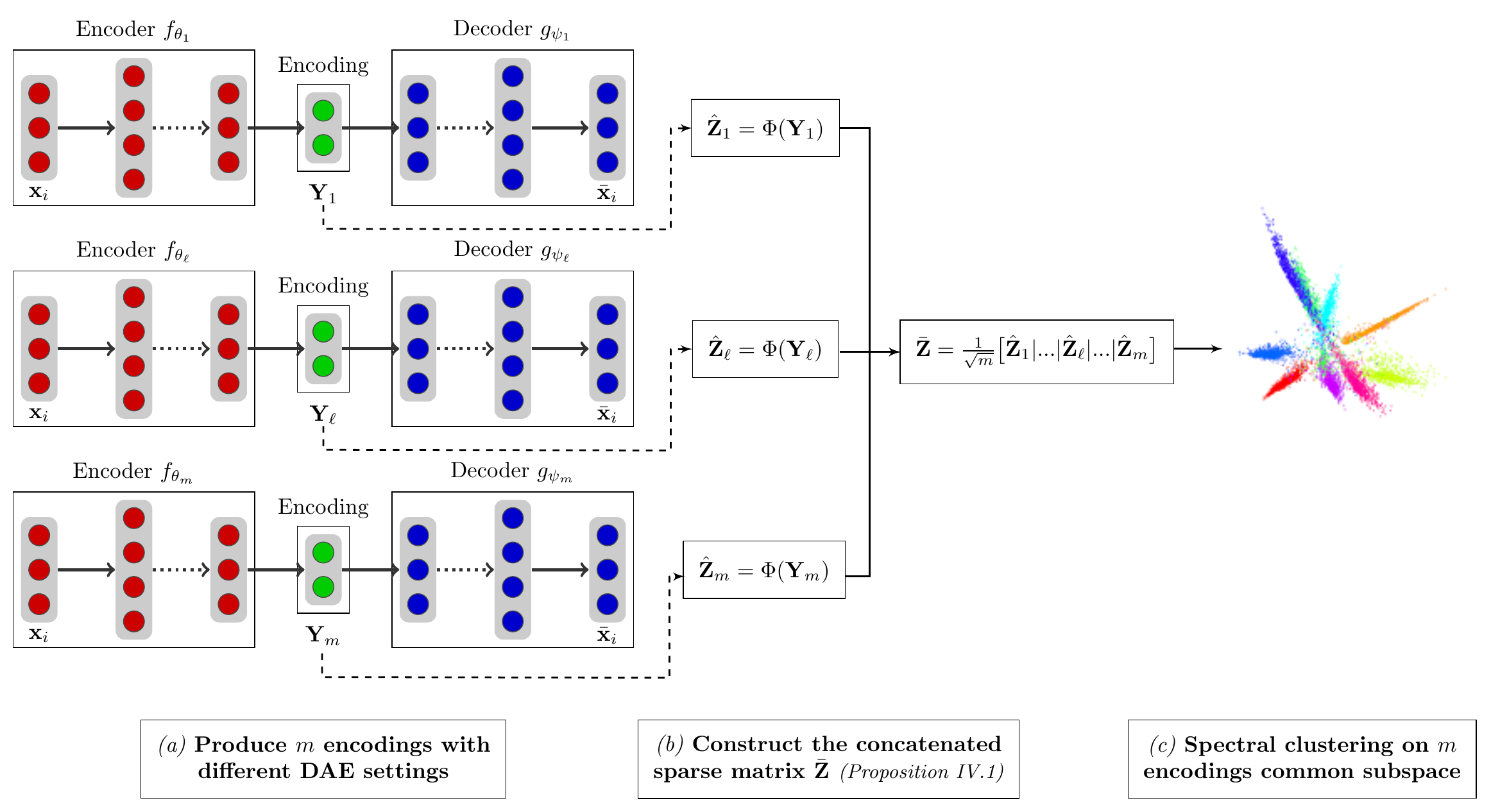}
\caption{{\bf Scheme of SC-EDAE.} {\footnotesize{The {\tt SC-EDAE} algorithm computes first $m$ encodings from DAE with different hyperparameters settings $(a)$, then generates $m$ sparse affinity matrix, $\{\hat{\mathbf{Z}}_\ell\}_{\ell \in [1,m]}$, that are concatenated in $\mathbf{\bar{Z}}$ $(b)$, and finally performs a {\tt SVD} on the ensemble graph affinity matrix $\bar{\mathbf{Z}}$ $(c)$.}}}
\label{fig:DAE-LSC_scheme}
\end{figure*}

The originality and efficiency of our ensemble method hinges on the replacement of a costly eigendecomposition on $\bar{\bS} \in \mathbbm{R}^{n \times n}$ by an eigendecomposition on a low-dimensional and sparse matrix $\mathbf{\bar{Z}}$ $\in \mathbb{R}^{n \times \sum_{j=1}^{m} \ell_{j}}$, with $\sum_{j=1}^{m} \ell_{j}\ll n$ (Alg.~\ref{Alg_sc-edae}, step (c)). In particular, the sparsity of $\mathbf{\bar{Z}}$ enables the use of fast iterative and partial eigenvalue decomposition.

\section{Experiments}
\label{sec:experiments}
\subsection{Deep autoencoders settings}
\label{sec:DAE_settings}
For our experiments, we trained fully connected autoencoders with an encoder $f_\theta$ of three hidden layers of size 50, 75 or 100 for synthetic datasets ({\tt Tetra}, {\tt Chainlink} and {\tt Lsun}; Section~\ref{sec:synth_data}), and three hidden layers of size 500, 750 or 1000 for real datasets ({\tt MNIST}, {\tt PenDigits} and {\tt USPS}; Section~\ref{sec:real_data}), as suggested by Bengio {\it et al.}~\citep{bengio2007greedy}, in all possible orders. The decoder part $g_\psi$ mirrors the encoder stage $f_\theta$. For each DAE architecture (e.g., $\{750-500-1000\}$, $\{100-50-75\}$), 5 encodings were generated with 50, 100, 150, 200 and 250 epochs for real datasets and 200 epochs for synthetic datasets. The weights initialization follows the Glorot's approach~\citep{glorot2010understanding} and all encoder/decoder pairs used rectified linears units (ReLUs), except for the output layer which requires a sigmoid function. The autoencoder data are systematically $L_2$ normalized. We configure the autoencoders using the {\tt Keras tensorflow} Python package, and compile the neural network with binary cross-entropy loss and Adam optimizer~\citep{reddi2018convergence} with the default {\tt Keras} parameters.

\subsection{SC-EDAE ensemble strategy}
\label{sec:SCEDAE_ensStrategy}
The ensemble strategy of {\tt SC-EDAE} exploits the encodings $\{\bY_\ell\}_{\ell \in [1,m]}$vwhich are generated with either {\it (i)} $m$ different DAE initializations or $m$ different DAE epochs number 
in association with one DAE structure (e.g. $d$--500--1000--750--$e$, with $d$ and $e$ the input and encoding layers width resp.), or {\it (ii)}~$m$ DAE with different structures for the same number of landmarks and epochs. In both cases, the {\tt SC-EDAE} strategy enables to compute the $m$ different sparse affinity matrices $\{\hat{\bZ}_\ell\}_{\ell \in [1,m]}$ (Eq.~\ref{eqanchor}) and, following Proposition~\ref{propo1}, generate the ensemble affinity matrix $\mathbf{\bar{Z}}$ (Eq.~\ref{eqzbar}).

\subsection{Synthetic datasets}
\label{sec:synth_data}
As a first step, we focus on synthetic datasets to illustrate the {\tt SC-EDAE} algorithm and show the class-separability information embedded in the left singular vectors matrix of $\bar{\mathbf{Z}}$, noted as $\mathbf{B}^\star$ (Prop.~\ref{propo1} and Alg.\ref{Alg_sc-edae}).
We used generated synthetic data sets selected from the Fundamental Clustering Problem Suite (FCPS)\footnote{The suite can be downloaded from the website of the author: http://www.uni-marburg.de/fb12/datenbionik/data}.
FCPS yields some hard clustering problems, a short description of {\tt Tetra}, {\tt Chainlink} and {\tt Lsun} FCPS data sets and the inherent problems related to clustering are given in Table~\ref{tab:FCPS}. Following the experiments on synthetic data proposed by Yang {\it et al.}~\cite{YangFSH17}, we transformed the low-dimensional FCPS data, $\bh_i \in \mathbb{R}^2$ or $\mathbb{R}^3$, in high-dimensional datapoints, $\boldsymbol{{x}_i} \in \mathbb{R}^{100}$. Specifically, the $\bx_i$ are transformed based on the following equation,
\begin{equation}\label{eqtransform_9}
  \bx_i = \sigma(\mathbf{U}\sigma(\mathbf{W}\bh_i)) 
\end{equation}
where the entries of matrices $\mathbf{W} \in \mathbb{R}^{10\times2}$ and $\mathbf{U} \in \mathbb{R}^{100\times10}$ follow the zero-mean unit-variance i.i.d. Gaussian distribution, and the sigmoid function $\sigma(.)$ introduces nonlinearity. 
\begin{table}[!h]
\centering
\setlength\tabcolsep{4pt}
\caption{Description of the used FCPS data sets.}
\scriptsize
\begin{tabular}{|l|c|c|c|c|}
\hline
\text{Data sets} &\multicolumn{4}{c|}{\text{Characteristics}}\\
                    \cline{2-5} 
                     & \multicolumn{1}{c|}{\text{Samples}} & \multicolumn{1}{c|}{\text{Features}} & \multicolumn{1}{c|}{\text{Clusters}}  & \multicolumn{1}{c|} {Main Problem} \\ 
\hline
\text{Tetra}      &  \multicolumn{1}{c|}{400}   & \multicolumn{1}{c|}{3}  & \multicolumn{1}{c|}{4}  & \multicolumn{1}{l|}{inner vs inter cluster distances} \\
\hline
\text{Chainlink}   &  \multicolumn{1}{c|}{1000}  & \multicolumn{1}{c|}{3}  & \multicolumn{1}{c|}{2}  & \multicolumn{1}{l|}{not linear separable} \\
\hline
\text{Lsun}       &  \multicolumn{1}{c|}{400}   & \multicolumn{1}{c|}{2}  & \multicolumn{1}{c|}{3}  & \multicolumn{1}{l|}{different variances} \\
\hline
\end{tabular}
\label{tab:FCPS}
\end{table}

\begin{figure*}[!h]
\centering
\setlength\tabcolsep{4pt}
\begin{tabular}[t]{lccc}
& \makecell{{\bf (a) Tetra}} & \makecell{{\bf (b) Chainlink}} & \makecell{{\bf (c) Lsun}}\\
\vspace{-0.6cm}
\makecell{\bf \footnotesize Original \\\bf \footnotesize data} & \makecell{\\ \includegraphics[height=2.9cm,width=2.9cm]{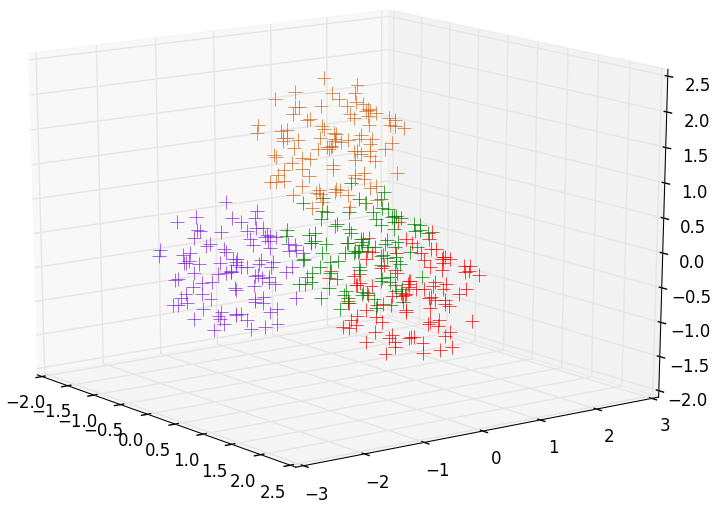} \\ }   & \makecell{\\\includegraphics[height=2.9cm,width=2.9cm]{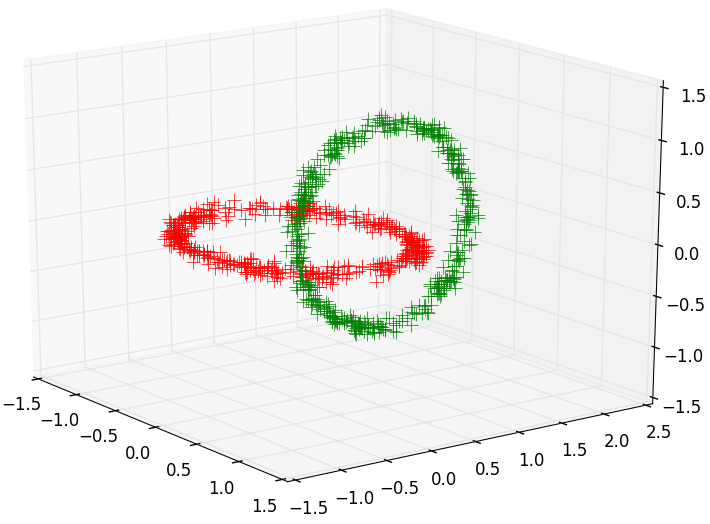} \\} & 
\makecell{\\ \includegraphics[height=2.9cm,width=2.9cm]{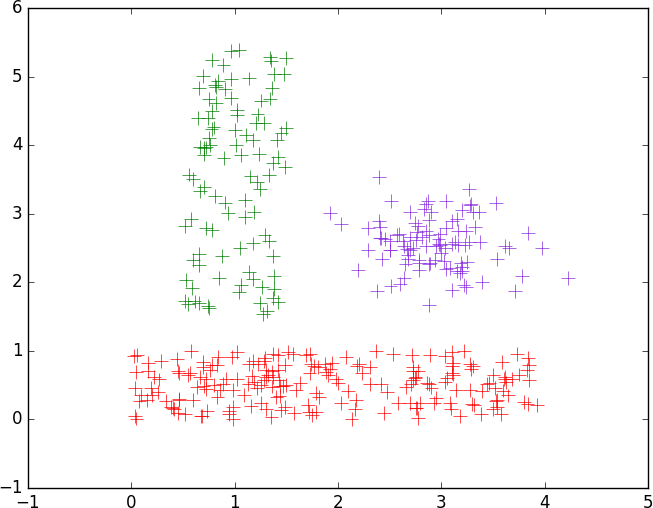} \\}\\
\makecell{\bf \footnotesize SC-EDAE\\ \bf \footnotesize Embeddings} & \makecell{\\ \includegraphics[height=2.9cm,width=2.9cm]{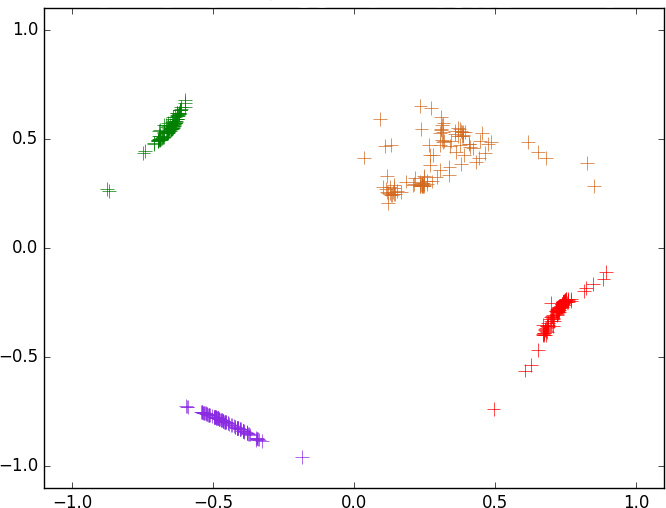} \\ }   & \makecell{\\\includegraphics[height=2.9cm,width=2.9cm]{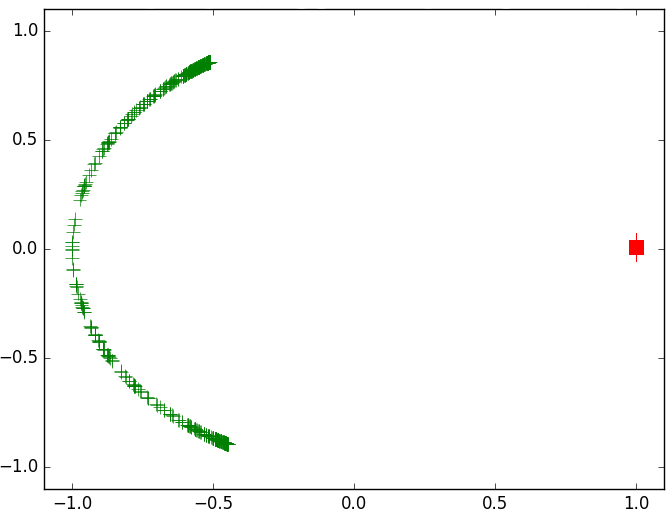} \\} & 
\makecell{\\ \includegraphics[height=2.9cm,width=2.9cm]{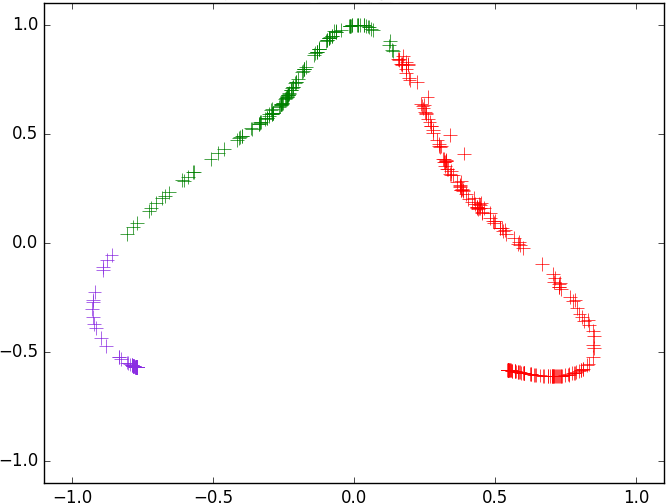} \\}\\
& \makecell{\scriptsize{SC-EDAE $acc = 1.00$}} & \makecell{\scriptsize{SC-EDAE $acc = 1.00$}} & \makecell{\scriptsize{SC-EDAE $acc = 0.90$}}\\
 \end{tabular}
\caption{{Visualization of the {\tt SC-EDAE} embeddings on {\tt Tetra}, {\tt Chainlink} and {\tt Lsun} datasets} {The two first components of $\bB$ (Alg.\ref{Alg_sc-edae}, step (c)) gives a visualization of the datapoints separability with the {\tt SC-EDAE} method. Colors indicate the predicted labels. }}
\label{fig:FCPS_synth9}
\end{figure*}

\subsection{Real datasets}
\label{sec:real_data}
Our {\tt SC-EDAE} algorithm (Alg.\ref{Alg_sc-edae}) is fully evaluated on three image datasets, namely {\tt MNIST} (Modified National Institute of Standards and Technology)~\cite{lecun1998gradient}, {\tt PenDigits} (Pen-Based Recognition of Handwritten Digits)~\cite{alimoglu1996methods} and {\tt USPS} (U.S. Postal Service)~\cite{vapnik1998statistical} and their DAE encodings (see Section~\ref{sec:DAE_settings} for details on DAE structure).

\begin{itemize}
	\item[] {\bf MNIST}~\cite{lecun1998gradient} The database is loaded from the {\tt Keras} Python package. The training and testing sets contain respectively $60,000$ and $10,000$ images of size $28\times28$ of the integers in range $0-9$. The images are of grayscale levels rescaled within $[0,1]$ by dividing by $255$.
    \item[] {\bf PenDigits}~\cite{alimoglu1996methods} The training and testing sets contain respectively $7,494$ and $3,498$ images of size $16\times16$ of the integers in range $0-9$. The images with 16 numeric attributes rescaled within $[0,1]$ by dividing by $100$.
    \item[] {\bf USPS}~\cite{vapnik1998statistical} The database is prepared as proposed in~\cite{guo2017improved}
    and contains $9,298$ images of size $16\times16$ pixels of the $10 \mhyphen$ digits (integers in range $0-9$) rescaled within $[0,1]$.
\end{itemize}

The classes distribution for each dataset is given in Table~\ref{tab:class_dist}. {\tt MNIST} and {\tt PenDigits} appear as balanced-class datasets while {\tt USPS} has an imbalanced distribution.
\begin{table}
\centering
\caption{Class distribution for {\tt MNIST}, {\tt PenDigits} and {\tt USPS} datasets.}
\setlength\tabcolsep{4pt}
\scriptsize
\begin{tabular}{|l|r|r|r|r|r|r|r|r|r|r| }
\hline
& \multicolumn{1}{c}{0} & \multicolumn{1}{|c}{1} & \multicolumn{1}{|c}{2} & \multicolumn{1}{|c}{3} & \multicolumn{1}{|c}{4} & \multicolumn{1}{|c}{5} & \multicolumn{1}{|c}{6} & \multicolumn{1}{|c}{7} & \multicolumn{1}{|c}{8} & \multicolumn{1}{|c|}{9}\\
\hline
MNIST & 5923 & 6742 & 5958 & 6131 & 5842 & 5421 & 5918 & 6265 & 5851 & 5949\\
PenDigits & 780 & 779 & 780 & 719 & 780 & 720 & 720 & 778 & 719 & 719\\
USPS & 1194 & 1005 & 731 & 658 & 652 & 556 & 664 & 645 & 542 & 644\\
\hline
\end{tabular}
\label{tab:class_dist}
\end{table}

\subsection{Experiment results}
\label{sec:exp_res}
\subsubsection{Evaluation on synthetic data}
Synthetic data enable us to easily explore the separability capacity of the embeddings matrix $\mathbf{B}$. For the experiments related to synthetic data, {\tt SC-EDAE} is used in its ensemble structure version, with $m=6$ encodings from different structures,
and the number of landmarks is set to 100. Applying {\tt SC-EDAE} on the data sets {\tt Tetra}, {\tt Chainlink} and {\tt Lsun}, we note that the 2D representations of the obtained clusters reflect the real cluster structure (Fig.~\ref{fig:FCPS_synth9} a, b, c; projection on the two first components of the matrix $\mathbf{B}$ as computed in Alg.\ref{Alg_sc-edae}, step c). The {\tt SC-EDAE} accuracy is of 1.00 for {\tt Tetra} and {\tt Chainlink}, and 0.90 for {\tt Lsun}. 
The colored labels correspond to the predicted clusters. Complementary tests with different transformation functions confirm this trend (see annexes, Section~\ref{sec:sup_exp_synth}).

\subsubsection{Baseline evaluations on real data}

As baseline, we first evaluate \emph{k-means} and \emph{LSC}~\cite{Chen11LandmarkSpectral} on the three real datasets. The \emph{kmeans}$_{++}$ approach corresponds to the {\tt scikit-learn} Python package \emph{k-means} implementation with the default parameters and \emph{kmeans}$_{++}$ initialization scheme~\cite{arthur2007k}. We implemented the \emph{LSC} method in Python, following the Matlab implementation proposed in~\cite{Chen11LandmarkSpectral}, and kept the same default parameters. The \emph{LSC} landmarks initialization is done with \emph{k-means}, which has been shown to provide better accuracy results than the random initialization~\cite{Chen11LandmarkSpectral,BanijamaliG17}. We consider landmarks number within $100$ and $1000$, by step of $100$. The evaluations are done either on the original datasets (Table~\ref{tab:ref_accuracy}, columns \emph{LSC} and \emph{kmeans}$_{++}$ or on the encodings (Table~\ref{tab:ref_accuracy}, columns {\it DAE-LSC} and {\it DAE-kmeans$_{++}$}). The accuracy reported for \emph{LSC} and {\it k-means$_{++}$} corresponds to the mean over $10$ clustering replicates on the original datasets, over all epoch and landmark numbers. The accuracy reported for {\it DAE-LSC} and {\it DAE-kmeans$_{++}$} corresponds to an average over 50 replicates (10 replicates on each of the 5 encodings per DAE structure), over all epoch and landmark numbers (see annexes for complementary results per DAE structure, Section~\ref{sec:sup_exp_real}). 

As can be seen from Table~\ref{tab:ref_accuracy} and already reported in~\cite{Chen11LandmarkSpectral}, \emph{LSC} outperforms \emph{kmeans}$_{++}$ for the clustering task on the three datasets (bold values, columns \emph{LSC} and {\it kmeans$_{++}$}), yet with larger standard deviations. The same trend is observed when applying \emph{LSC} and {\it kmeans$_{++}$} on encodings, with standard deviations of similar magnitude for both clustering methods (bold values, columns {\it DAE-LSC} and {\it DAE-kmeans$_{++}$}).  

\begin{table*}[!h]
\centering
\caption{{\bf Mean clustering accuracy for \emph{LSC} and \emph{k-means} on original real datasets and encodings:} {\footnotesize{Evaluations on MNIST, PenDigits, USPS data and their encodings. Bold values highlight the higher accuracy values.}}}
\scriptsize
\setlength\tabcolsep{1pt}
\begin{tabular}{|c|c|c||c|c|c|c| }
\hline
{\footnotesize{Data}} & {\footnotesize{LSC}} & {\footnotesize{kmeans$_{++}$}} & {\footnotesize{DAE structure}} & {\footnotesize{DAE-LSC}} & {\footnotesize{DAE-kmeans$_{++}$}}\\
\hline
\multirow{6}{*}{\footnotesize{MNIST}} & \multirow{6}{*}{{\bf 68.55} $\pm 2.25 $} & \multirow{6}{*}{55.13 $\pm 0.05$} & 500--750--1000 & 87.06 $\pm 8.27 $ & 76.33 $\pm 7.69 $\\
 & & & 500--1000--750 & 90.48 $\pm 5.20 $ & 79.22 $\pm 5.93 $\\
 & & & 750--500--1000 & 88.31 $\pm 5.46 $ & 77.71 $\pm 6.03 $ \\
 & & & 750--1000--500 & 90.30 $\pm 4.89 $ & 79.45 $\pm 5.81 $\\
 & & & 1000--500--750 & {\bf 91.54} $\pm 3.06 $ & 79.98 $\pm 5.98 $\\
 & & & 1000--750--500 & 90.96 $\pm 3.98$ & 77.70 $\pm 5.09 $\\
 \hline
\multirow{6}{*}{\footnotesize{PenDigits}} & \multirow{6}{*}{{\bf 80.17} $\pm 3.76 $} & \multirow{6}{*}{73.89 $\pm 3.97 $} & 500--750--1000 & {\bf 85.59} $\pm 2.34 $ & 73.64 $\pm 4.00 $\\
 & & & 500--1000--750 & 85.11 $\pm 3.15 $ & 74.67 $\pm 3.43 $\\
 & & & 750--500--1000 & 85.36 $\pm 2.91 $ & 73.47 $\pm 3.89 $\\
 & & & 750--1000--500 & 85.27 $\pm 2.92 $ & 74.64 $\pm 4.01 $\\
 & & & 1000--500--750 & 85.02 $\pm 2.72 $ & 74.20 $\pm 3.84 $\\
 & & & 1000--750--500 & 84.39 $\pm 3.04 $ & 73.78 $\pm 3.55 $\\
\hline
\multirow{6}{*}{\footnotesize{USPS}} & \multirow{6}{*}{{\bf 77.20} $\pm 1.49 $} & \multirow{6}{*}{68.36 $\pm 0.08$} & 500--750--1000 & 81.78 $\pm 8.08 $ & 72.85 $\pm 3.52 $\\
 & & & 500--1000--750 & {\bf 83.47} $\pm 7.40 $ & 73.44 $\pm 3.70 $\\
 & & & 750--500--1000 & 79.72 $\pm 6.21 $ & 72.46 $\pm 2.78 $\\
 & & & 750--1000--500 & 80.29 $\pm 5.70 $ & 73.80 $\pm 3.51 $\\
 & & & 1000--500--750 & 81.39 $\pm 4.46 $ & 74.07 $\pm 3.07 $\\
 & & & 1000--750--500 & 83.08 $\pm 5.64 $ & 72.41 $\pm 3.06 $\\
 \hline
\end{tabular}
\label{tab:ref_accuracy}
\end{table*}

The results from Table~\ref{tab:ref_accuracy} demonstrate that the simple combination of DAE and \emph{LSC} or \emph{k-means} already reaches higher accuracy and smaller standard deviations than without the autoencoder step. 
These results also show the advantage of associating the DAE encodings with the landmark-based representation over the \emph{k-means} approach for the clustering task (columns {\it DAE-LSC} and {\it DAE-kmeans$_{++}$}). In particular, the average accuracy for the {\tt MNIST} and {\tt USPS} datasets varies within $[87.06;91.54]$ and $[79.72;83.47]$ respectively for {\it DAE-LSC} and within $[77.70;79.98]$ and $[72.41;74.07]$ respectively for {\it DAE-kmeans$_{++}$}. 

Although the encodings generated by the deep autoencoder improve the clustering accuracy, finding {\it a priori} the most appropriate DAE structure remains a challenging task. The accuracy may also vary for different landmark and epoch numbers (see Table~\ref{tab:ens_accuracy_lm} and annexes Tables~\ref{tab:ref_ari_nmi}~\&~\ref{tab:ens_ari_nmi}). As will be seen in the following sections, the ensemble strategy of {\tt SC-EDAE} provides a straightforward way to alleviate these issues and avoid arbitrary DAE hyperparameters setting.

\subsubsection{SC-EDAE ensemble evaluations}
\label{sec:DAE-LSC_ens_results}

The Table~\ref{tab:ens_accuracy} summarizes the performance of our \emph{LSC}-based ensemble approach in the two cases detailed in section~\ref{sec:SCEDAE_ensStrategy}. Specifically, the columns {\it Ens.Init.} and {\it Ens.Ep.} 
indicate the clustering accuracy for the case {\it (i)} with an ensemble approach on the DAE weights initialization ({\it Ens.Init.}, $m=5$) and the DAE training epoch numbers ({\it Ens.Ep.}, $m=5$). 
The clustering accuracy values for the ensemble approach on various DAE structures, {\it i.e.} case ({\it ii}), is provided in the column {\it Ens.Struct.} ($m=6$). 

The {\tt SC-EDAE} ensemble strategy provides higher clustering accuracy as compare to the baseline evaluations (Table~\ref{tab:ref_accuracy}).
In particular, the mean accuracy values obtained with the ensemble strategy for {\tt MNIST}, {\tt PenDigits} and {\tt USPS} can reach, $95.33 \pm 0.07$, $87.28 \pm 0.48$ and $85.22 \pm {2.14 }$ respectively, {\it vs.} $91.54 \pm 3.06$, $85.59 \pm {2.34 }$ and $83.47 \pm {7.40 }$ (Table~\ref{tab:ref_accuracy}).

The {\tt SC-EDAE} ensemble approach on the DAE structures ({\it Ens.Struct.}) enables also to reach higher accuracy as compare to the baseline evaluations for {\tt MNIST} ($93.23 \pm 0.28$ {\it vs.} $91.54 \pm 3.06 $) and {\tt PenDigits} ($86.44 \pm 1.42 $ {\it vs.} $85.59 \pm {2.34 }$), but with the added benefit of avoiding the arbitrary choice of a particular DAE structure. The {\tt SC-EDAE} results for {\tt USPS} with an ensemble on several structures are lower than our reference evaluations ($81.78 \pm {3.61 }$ {\it vs.} $83.47 \pm {7.40 }$), yet the accuracy value remains fairly high with lower standard deviation. 

\begin{table}[!h]
\centering
\caption{{\bf Mean clustering accuracy for {\tt SC-EDAE}, ensemble on initializations, epochs number and structures:}
Bold values highlight the higher accuracy values.}
\scriptsize
\setlength\tabcolsep{4pt}
\begin{tabular}{|c|c||c|c|c| }
\hline
{\footnotesize{Dataset}} & {\footnotesize{DAE structure}} & {\footnotesize{Ens.Init.}} & {\footnotesize{Ens.Ep.}} & {\footnotesize{Ens.Struct.}}\\
\hline
\multirow{6}{*}{\footnotesize{MNIST}} & 500--750--1000 & 89.19
$\pm 0.41$ & 85.54 $\pm 4.30 $ & \multirow{6}{*}{93.23 $\pm 2.84 $}\\
 & 500--1000--750 & {\bf 95.33} $\pm 0.07$ & 94.34 $\pm 2.68 $ &\\
 & 750--500--1000 & 92.15 $\pm 0.25$ & 92.03 $\pm 3.87 $ &\\
 & 750--1000--500 & 92.65 $\pm 0.13$ & 92.26 $\pm 3.71 $ &\\
 & 1000--500--750 & 94.28 $\pm 0.20$ & 94.57 $\pm 1.48 $ &\\
 & 1000--750--500 & 93.87 $\pm 0.38$ & {\bf 95.25} $\pm 0.59$ &\\
 \hline
\multirow{6}{*}{\footnotesize{PenDigits}} & \scriptsize{$500$-$750$-$1000$} & {\bf 86.80}
 $\pm 0.74$ & 87.08 $\pm 1.10 $ & \multirow{6}{*}{86.44 $\pm 1.42$}\\
 & 500--1000--750 & 85.95 $\pm 0.73$ & 86.69 $\pm 1.33 $ &\\
 & 750--500--1000 & 86.69 $\pm 0.87$ & 87.27 $\pm 0.60$ &\\
 & 750--1000--500 & 86.48 $\pm 1.09 $ & 86.91 $\pm 1.01 $ &\\
 & 1000--500--750 & 86.75 $\pm 6.40$ & 86.96 $\pm 8.10$ &\\
 & 1000--750--500 & 86.66 $\pm 9.50$ & {\bf 87.28} $\pm 0.48$ &\\
\hline
\multirow{6}{*}{\footnotesize{USPS}} & 500--750--1000 & 80.07 $\pm 1.95 $ & 81.36 $\pm 5.09 $ & \multirow{6}{*}{81.78 $\pm 3.61 $}\\
 & 500--1000--750 & 80.54 $\pm 0.77$ & 82.06 $\pm 3.54 $ &\\
 & 750--500--1000 & 79.49 $\pm 1.19 $ & 81.10 $\pm 3.86 $ &\\
 & 750--1000--500 & 79.29 $\pm 1.05 $ & 79.88 $\pm 2.69 $ &\\
 & 1000--500--750 & 84.12 $\pm 1.80 $ & 81.89 $\pm 3.21 $ &\\
 & 1000--750--500 & {\bf 85.22} $\pm 2.14 $ & {\bf 84.96} $\pm 3.29 $ &\\
 \hline
\end{tabular}
\label{tab:ens_accuracy}
\end{table}
While the {\tt SC-EDAE} method aims at providing an ensemble strategy for the deep architecture settings ({\it Ens.Init.}, {\it Ens.Ep.} and {\it Ens.Struct.}, Table~\ref{tab:ens_accuracy}), it relies also on the \emph{LSC} idea which depends on the number of landmarks. We studied the possibility of an ensemble on the number of landmarks ($m=5$). As can be seen from Table~\ref{tab:ens_accuracy_lm}, which provides mean accuracy on $10$ replicates, the ensemble strategy enables again to reach high accuracy values as compared to our baseline evaluations. The results still remain dependent from the DAE structure type, in particular for {\tt MNIST} and {\tt USPS}, and we would therefore recommend to use {\tt SC-EDAE} in its ensemble structure version ({\it ie.}, {\it Ens.Struct.}).

\begin{table}[!h]
\centering
\caption{{\bf Mean clustering accuracy for  {\tt SC-EDAE}, ensemble on landmarks:} {\footnotesize{
Bold values highlight the higher accuracy values}}.}
\setlength\tabcolsep{4pt}
\scriptsize
\begin{tabular}{|c||c|c|c|}
\hline
DAE structure & MNIST & PenDigits & USPS\\
\hline
500--750--1000 & 88.84 $\pm 1.22 $ & {\bf 87.31} $\pm 1.13 $ & 82.17 $\pm 3.79 $\\
500--1000--750 & {\bf 95.35} $\pm 0.20$ & 87.21 $\pm 0.36$ & 81.96 $\pm 2.74 $\\
750--500--1000 & 92.48 $\pm 1.27 $ & 87.16 $\pm 0.99$ & 80.61 $\pm 3.46 $\\
750--1000--500 & 92.53 $\pm 0.76$ & 87.09 $\pm 0.95$ & 80.30 $\pm 1.26 $\\
1000--500--750 & 93.76$\pm 1.14 $ & 86.67 $\pm 1.40 $ & 86.35 $\pm 2.62 $\\
1000--750--500 & 95.08 $\pm 0.17$ & 87.13 $\pm 1.26 $ & {\bf 87.32} $\pm 4.85 $\\
 \hline
\end{tabular}
\label{tab:ens_accuracy_lm}
\end{table}

\subsection{Evaluation in terms of NMI and ARI }

Evaluating clustering results is not 
a trivial task. The clustering accuracy is not always a reliable measure when the clusters are not balanced and the number of clusters is high. To better appreciate the quality of our approach, in the sequel we retain two widely used measures to assess the quality of clustering, namely the Normalized Mutual Information  \cite{Strehl:2003} and the Adjusted Rand Index \citep{steinley2004properties}. Intuitively, NMI quantifies how much the estimated clustering is informative about the true clustering, while the ARI measures the degree of agreement between an estimated clustering and a reference clustering. Higher NMI/ARI is better.

We report in Figure~\ref{fig:boxplots_USPS_ARI} the ARI and NMI values for the three real datasets ({\tt MNIST}, {\tt PenDigits} and {\tt USPS}). The ARI and NMI values are given for the baseline evaluations ({\it DAE-kmeans$_{++}$} and {\it DAE-LSC}; average results over 10 runs), and the various ensemble versions of {\tt SC-EDAE} ({\it Ens.Init}, {\it Ens.Ep.} and {\it Ens.Struct.}; average results over 10 runs for each of the 5 different encodings). The ensemble paradigm of {\tt SC-EDAE} ensures high ARI and NMI results with low standard deviations for all real datasets, even for {\tt USPS} which is an imbalanced-class dataset (Fig.~\ref{fig:boxplots_USPS_ARI}, green boxplots). 

%
\begin{figure*}[!h]
\centering
\setlength\tabcolsep{4pt}
\begin{tabular}{ccc}
\makecell{\\ \includegraphics[height=2.9cm,width=4.1cm]{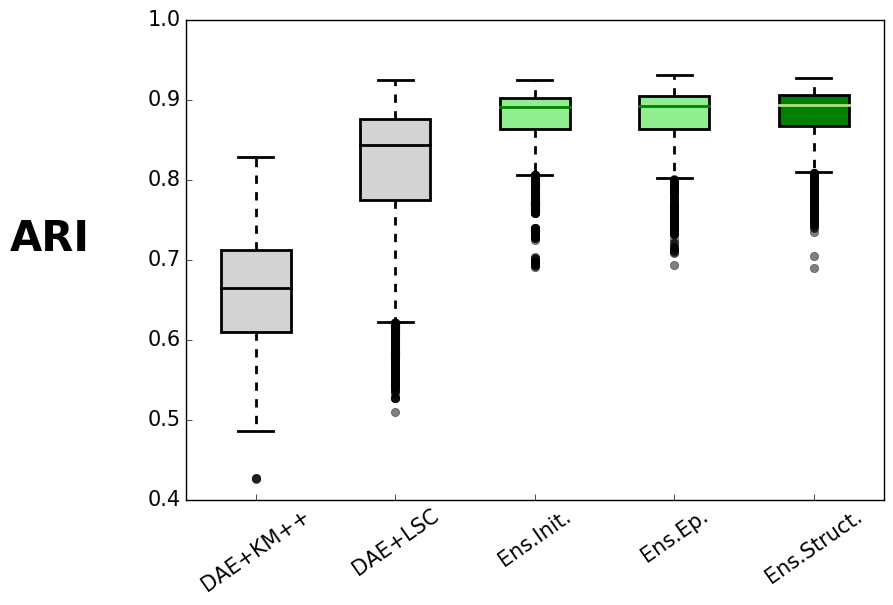}}  &  
\makecell{\\ \includegraphics[height=2.8cm,width=3.5cm]{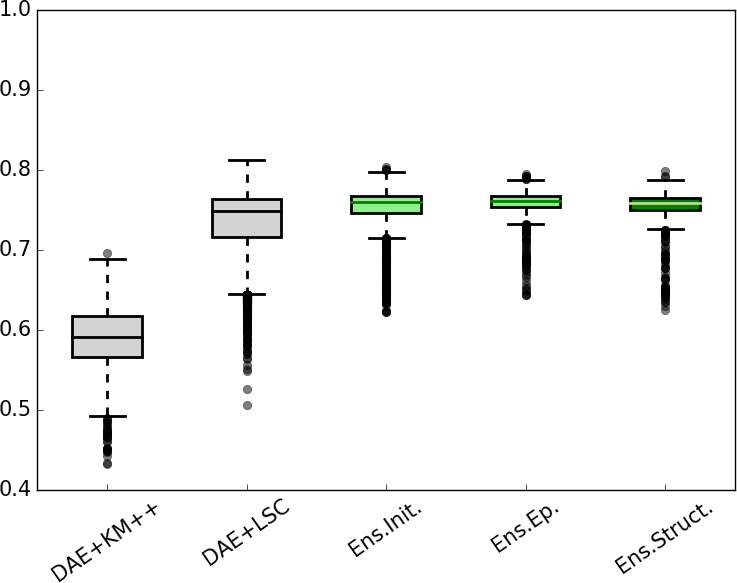}} & 
\makecell{\\ \includegraphics[height=2.8cm,width=3.5cm]{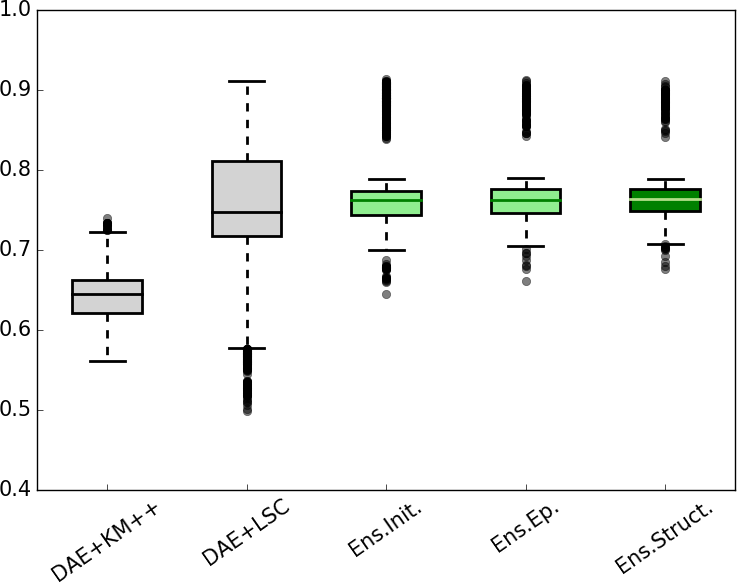} } \\
\includegraphics[height=2.9cm,width=4.1cm]{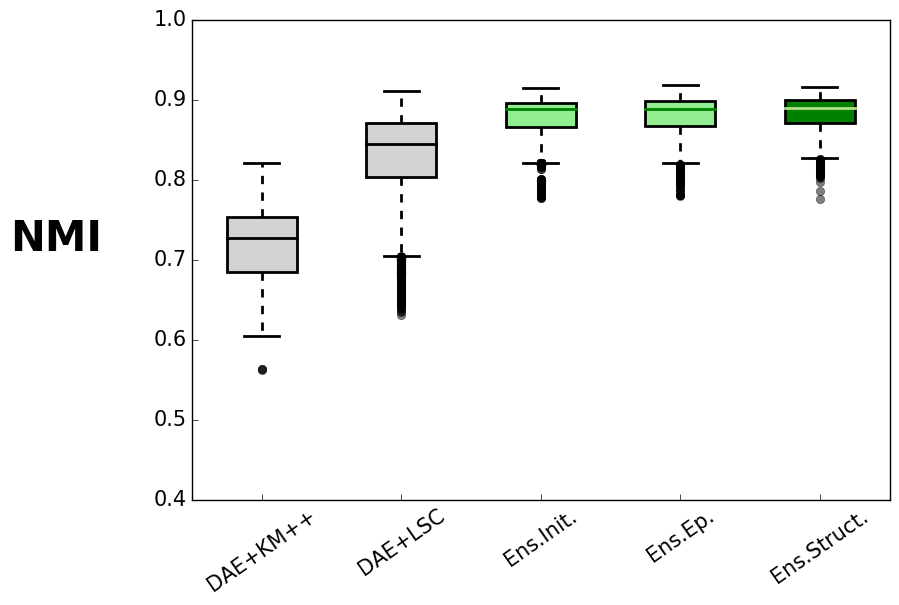} & \includegraphics[height=2.8cm,width=3.5cm]{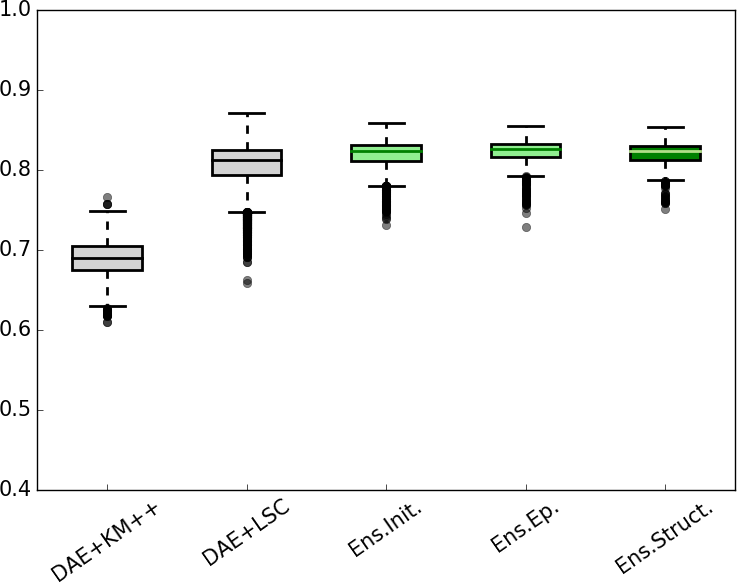} & \includegraphics[height=2.8cm,width=3.5cm]{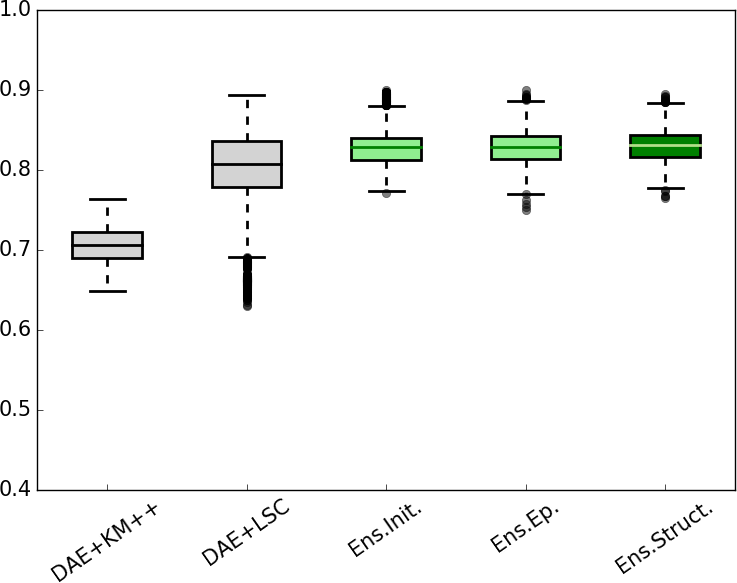} \\
{\footnotesize{\bf MNIST}} & {\footnotesize{\bf PenDigits}} & {\footnotesize{\bf USPS}}\\
\end{tabular}
\caption{\footnotesize{
Comparison of Adjusted Rand Index (ARI) and Normalized Mutual Information (NMI) for our {\tt SC-EDAE} approach (ensemble on initialization, epochs and structures; 10 runs) and baseline methods (combination of deep autoencoders and \emph{k-means} or \emph{LSC}; 10 runs for each of the 5 encodings).}}
\label{fig:boxplots_USPS_ARI}
\end{figure*}

We also detail the ARI and NMI evaluations per DAE structure in annexes, Tables~\ref{tab:ref_ari_nmi}~\&~\ref{tab:ens_ari_nmi}. These supplementary results highlight the strong influence of a particular DAE structure on the ARI and NMI values. As an example, the ARI minimal and maximal values for {\it DAE-LSC} are $73.66$ and $77.75$ respectively for {\tt USPS}, a difference of 4.09 (Table~\ref{tab:ref_ari_nmi}). Another striking example can be found for the {\tt SC-EDAE} in its ensemble initialization version ({\it Ens.Init.}) applied to {\tt MNIST}, where the ARI values fluctuate within a $[81.87;90.17]$ (Table~\ref{tab:ens_ari_nmi}). Based on these evaluations, and as already mentioned (Section~\ref{sec:SCEDAE_ensStrategy}), we would recommend to use {\tt SC-EDAE} in its ensemble structure version (i.e., {\it Ens.Struct.}) to alleviate the issue of the DAE structure choice.

\subsection{Comparison to deep k-means variants}

Several strategies that use  deep learning algorithm and \emph{k-means} approaches, sequentially or jointly, have demonstrated accuracy improvement on the clustering task. Among these methods, two approaches can now be considered as state-of-the-art methods, namely IDEC (Improved Deep Embedded Clustering)~\cite{Guo:2017} and DCN (Deep Clustering Network)~\cite{YangFSH17}. Very recently, the DKM (Deep \emph{k-means}) algorithm, which applies a \emph{k-means} in an AE embedding space, outperformed these approaches~\cite{fard2018deep}. 

\begin{table}[!h]
\centering
\caption{{\bf Mean clustering accuracy and NMI comparison with deep k-means variants:} {\footnotesize Mean accuracy and NMI for MNIST and USPS over $10$ replicates with {\tt SC-EDAE} and comparison to baselines and state-of-the-art approaches. Bold values highlight the higher accuracy values}.}
\setlength\tabcolsep{4pt}
\scriptsize
\begin{tabular}{|l|cc|cc| }
\hline
\multirow{2}{*}{\footnotesize{Model}} & \multicolumn{2}{c}{\footnotesize{MNIST}} & \multicolumn{2}{|c|}{\footnotesize{USPS}}\\
 & \footnotesize{ACC} & \footnotesize{NMI} & \footnotesize{ACC} & \footnotesize{NMI}\\
\hline
\multicolumn{5}{|c|}{\footnotesize{\bf baselines}}\\
\hline
 \scriptsize{kmeans$_{++}$} & 55.13 $\pm 0.05$ & 52.89 $\pm 0.02$ & 68.36 $\pm 0.08$ & 65.67 $\pm 0.10$\\
 \scriptsize{LSC} & 68.55 $\pm 2.25$ & 70.54 $\pm 0.83$ & 77.20 $\pm 1.49$ & 79.48 $\pm 0.90$\\
 \hline
 \scriptsize{DAE+kmeans$_{++}$} & 78.40 $\pm 6.09$ & 71.97 $\pm 4.13$ & 73.17 $\pm 3.27$ & 70.48 $\pm 1.84$\\
 \scriptsize{DAE+LSC} & 89.78 $\pm 5.14$ & 83.06 $\pm 4.38$ & 81.62 $\pm 6.25$ & 80.44 $\pm 3.39$\\
\hline
\multicolumn{5}{|c|}{\footnotesize{\bf no pretraining required}}\\
\hline
 \scriptsize{SC-EDAE Ens.Init.} & 92.91 $\pm 0.24$ & 87.65 $\pm 0.18$ & 81.46 $\pm 1.48$ & 82.88 $\pm 0.59$\\
 \scriptsize{SC-EDAE Ens.Ep.} & 92.33 $\pm 2.77$ & 87.72 $\pm 2.42$ & {\bf 81.88} $\pm 3.62$ & 83.03 $\pm 1.88$\\
 \scriptsize{SC-EDAE Ens.Struct.} & {\bf 93.23} $\pm 2.84$ & {\bf 87.93} $\pm 2.27$ & 81.78 $\pm 3.61$ & {\bf 83.17} $\pm 1.96$\\
\hline
\multicolumn{5}{|c|}{\footnotesize{Deep clustering approaches without pretraining}~\scriptsize{\it (Fard et al. 2018)~\cite{fard2018deep}}}\\
\hline
\scriptsize{DCN$^{np}$} & 34.8 $\pm 3.0$ & 18.1 $\pm 1.0$ & 36.4 $\pm 3.5$ & 16.9 $\pm 1.3$ \\
\scriptsize{IDEC$^{np}$} & 61.8 $\pm 3.0$ & 62.2 $\pm 1.6$ & 53.9 $\pm 5.1$ & 50.0 $\pm 3.8$ \\
\scriptsize{DKM$^{a}$} & 82.3 $\pm 3.2$ & 78.0 $\pm 1.9$ & 75.5 $\pm 6.8$ & 73.0 $\pm 2.3$ \\
\hline
\multicolumn{5}{|c|}{\footnotesize{Deep clustering approaches with pretraining}~\scriptsize{\it (Fard et al. 2018)~\cite{fard2018deep}}}\\
\hline
\scriptsize{DCN$^{p}$} & 81.1 $\pm 1.9$ & 75.7 $\pm 1.1$ & 73.0 $\pm 0.8$ & 71.9 $\pm 1.2$ \\
\scriptsize{IDEC$^{p}$} & 85.7 $\pm 2.4$ & 86.4 $\pm 1.0$ & 75.2 $\pm 0.5$ & 74.9 $\pm 0.6$ \\
\scriptsize{DKM$^{p}$} & 84.0 $\pm 2.2$ & 79.6 $\pm 0.9$ & 75.7 $\pm 1.3$ & 77.6 $\pm 1.1$ \\
\hline
\end{tabular}
\label{tab:ens_accuracy_cmp}
\end{table}
We compare {\tt SC-EDAE} to these three methods and summaries these evaluations in Table~\ref{tab:ens_accuracy_cmp}. The last six rows of Table~\ref{tab:ens_accuracy_cmp} are directly extracted from the DKM authors study~\cite{fard2018deep}. The accuracy and NMI values of these six rows are an average over 10 runs. The other values correspond to our evaluations. Specifically, baseline results are given in the first four rows, and correspond to the clustering task via {\it k-means$_{++}$} or \emph{LSC} (average results over 10 runs), and via a combination of DAE and \emph{k-means} or \emph{LSC} (average results over 10 runs for each of the 5 different encodings). The {\tt SC-EDAE} rows gives the accuracy and NMI results for our ensemble method, with an ensemble over several initializations ({\it SC-EDAE Ens.Init.}), epoch numbers ({\it SC-EDAE Ens.Ep.}) and DAE architectures ({\it SC-EDAE Ens.Struct.}).

As can be seen from Table~\ref{tab:ens_accuracy_cmp}, while our {\tt SC-EDAE} approach does not require any pretraining, it outperforms the DCN and IDEC methods in there pretrained version (Table~\ref{tab:ens_accuracy_cmp}, DCN$^p$ and IDEC$^p$ results). The DKM method performs well with and without pretraining. Yet, our {\tt SC-EDAE} approach reaches higher accuracy and NMI results than the DKM approach with and without pretraining.

\subsection{Visualization of latent space}

We investigate the quality of the representation learned with {\tt SC-EDAE} and in particular the positive influence of the left singular vectors matrix of $\bar{\mathbf{Z}}$, $\mathbf{B}$ (Alg.\ref{Alg_sc-edae}, step c), on the clustering task. Specifically, we visualize the datapoints nearest-neighbor from the $\mathbf{B}$ matrix using the t-SNE visualization tool~\cite{maaten2008visualizing} that can project embeddings into two components ({\tt TSNE} Python version from the {\tt sklearn} package ). The results are given in Figure~\ref{fig:tSNE_B}. The t-SNE hyperparameters {\it perplexity}, {\it learning rate} and {\it number of iterations} are set to 40, 200 and 500 for {\tt MNIST}, and 25, 100 and 400 for {\tt PenDigits} and {\tt USPS}, following the recommendations and experimental setup of Maaten {\it et al.}~\cite{maaten2008visualizing}.  
\begin{figure*}[!h]
\vspace{-0.3cm}
\centering
\setlength\tabcolsep{1.5pt}
\begin{tabular}[t]{ccc}
\makecell{\\ \hspace{0cm} \includegraphics[height=3.5cm,width=3.8cm]{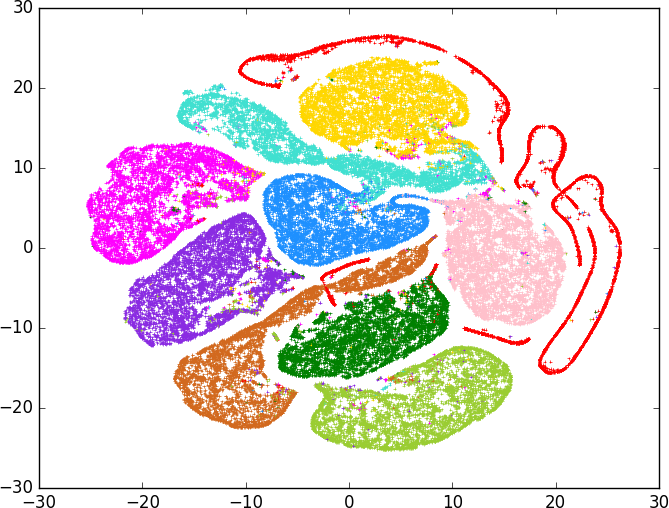} \\ }   & \makecell{\\ \hspace{0cm} \includegraphics[height=3.5cm,width=3.8cm]{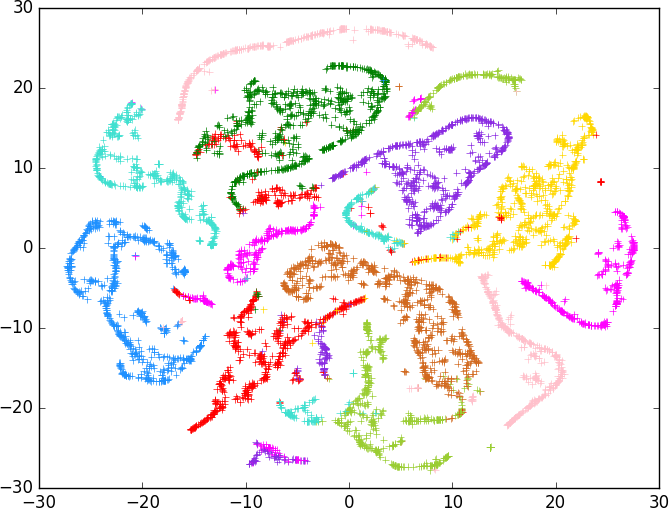} \\} & 
\makecell{\\ \hspace{0cm} \includegraphics[height=3.5cm,width=3.8cm]{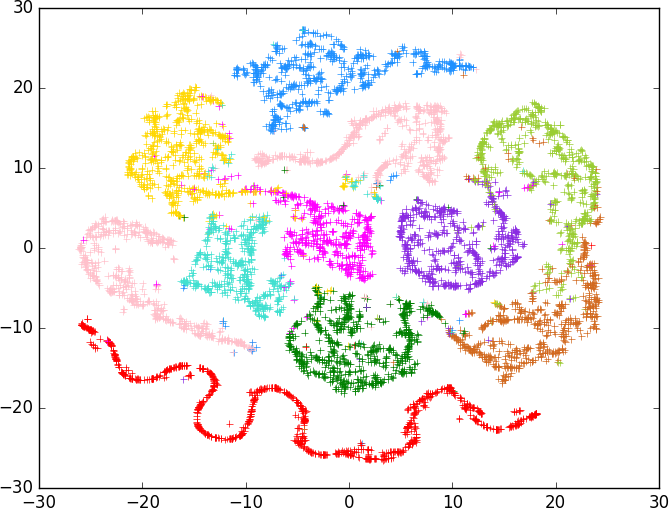} \\}\\
{\footnotesize{\bf MNIST}} & {\footnotesize{\bf PenDigits}} & {\footnotesize{\bf USPS}}\\
 \end{tabular}
 \vspace{-0.3cm}
 \caption{{ \footnotesize{t-SNE Vizualization of the embeddings $\bB$ from the {\tt SC-EDAE} approach on MNIST, PenDigits and USPS datasets. The t-SNE approach provides clustering visualization of the datapoints from the $\bB$ embeddings. Colors indicate the ground truth labels corresponding to the digits from 0 to 9.}}}
\label{fig:tSNE_B}
\vspace{-0.3cm}
\end{figure*}
For each dataset, we can observe clearly separated clusters. The ground truth labels nicely match the t-SNE datapoints gathering, highlighting the ability of {\tt SC-EDAE} to separate data according to the underlying classes. As already noticed in~\cite{maaten2008visualizing}, the t-SNE results obtained from the {\tt SC-EDAE} ensemble affinity matrix reflects the local structure of the data, such as the orientation of the ones, by showing elongated clusters (e.g., Fig.~\ref{fig:tSNE_B}, red cluster).
%

%
\section{Conclusion}\label{sec:discussion}
\label{sec:conclusion}

We report in this paper a novel clustering method that combines the advantages of deep learning, spectral clustering and ensemble strategy.
Several studies have proposed to associate, either sequentially or jointly, deep architecture and classical clustering methods to improve the partitioning
of large datasets. However, these methods are usually confronted to important issues related to well known challenges with neural networks, such as weight initialization
or structure settings. Our {\tt SC-EDAE} approach alleviates these issues by exploiting an ensemble procedure to combine several deep models before applying a spectral clustering; it is quite simple and can be framed in three steps:
\begin{itemize}
\item generate $m$ deep embeddings from the original data,
\item construct a sparse and low-dimensional ensemble affinity matrix based on anchors strategy,
\item apply spectral clustering on the common space shared by the $m$ encoding.
\end{itemize}
The experiments on real and synthetic datasets demonstrate the robustness and high performance of {\tt SC-EDAE} on image datasets. 
{\tt SC-EDAE} can be used in different versions with an ensemble on weights initialization, epoch numbers or deep architectures.
These variants provide higher accuracy, ARI and NMI results than state-of-the art methods. Most importantly, the high performance
of {\tt SC-EDAE} is obtained {\it without} any deep models pretraining.  

The proposed method also benefits from the anchors strategy. The anchors provide a sparse and low-dimensional
ensemble affinity matrix that ensures an efficient spectral clustering. As a complementary improvement, one could easily
implements the parallelization of the $m$ encodings computation in the first step of the {\tt SC-EDAE} procedure. Our experiments show that few 
different encodings already lead to significant performance improvement, yet more complex datasets could require larger amount
of various encodings, and such parallelization would facilitate the {\tt SC-EDAE} use. 

\pagebreak
\appendix
\section{Appendix}
\subsection{Supplementary experiments on synthetic data}
\label{sec:sup_exp_synth}
As proposed in~\cite{YangFSH17}, we provide two complementary examples of clustering with {\tt SC-EDAE} that demonstrate the ability of the $\mathbf{B}$ embeddings to correctly recover the underlying classes of a given dataset. We first consider the following two transformations, $\bx_i = \sigma(\sigma(\mathbf{W}\bh_i))^2$ and $\bx_i = \tan(\sigma(\mathbf{W}\bh_i))$.
The Figure \ref{fig:FCPS_synth10_and_11} shows the two first embeddings of $\bB$ obtained with the transformed data. This representation highlights the separability power of {\tt SC-EDAE}. The corresponding accuracy is $1.00$ for {\tt Tetra}, {\tt Chainlink} and {\tt Lsun}. For both supplementary transformation, we can observe patterns that are similar to clusters presented in the main text (Fig.~\ref{fig:FCPS_synth9}).
%
\begin{figure*}[!h]
\centering
\setlength\tabcolsep{1pt}

\begin{tabular}[t]{cccc}
 & \footnotesize{\bf Tetra} & \footnotesize{\bf Chainlink} & \footnotesize{\bf Lsun}\\
 \vspace{-0.cm}
$\scriptstyle \bx_i = \sigma(\sigma(\mathbf{W}\bh_i))^2$ & \makecell{\\ \includegraphics[height=3.2cm,width=3.2cm]{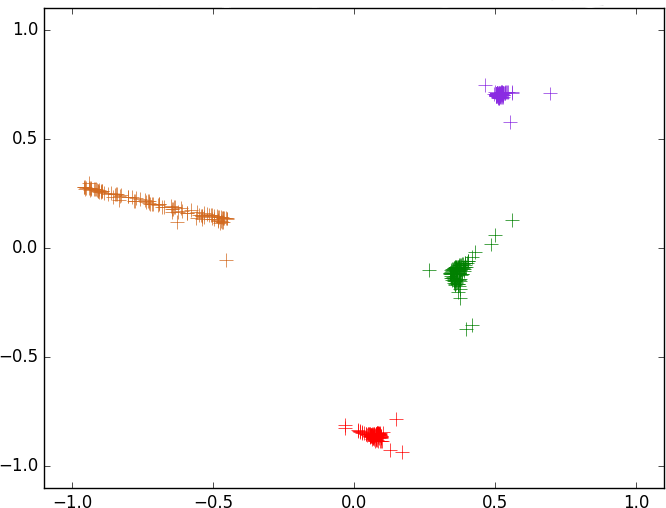} \\ } & \makecell{\\\includegraphics[height=3.2cm,width=3.2cm]{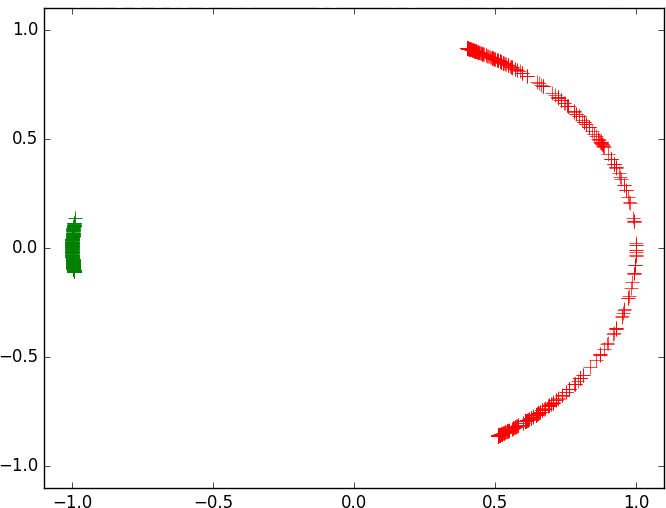} \\} & \makecell{\\ \includegraphics[height=3.2cm,width=3.2cm]{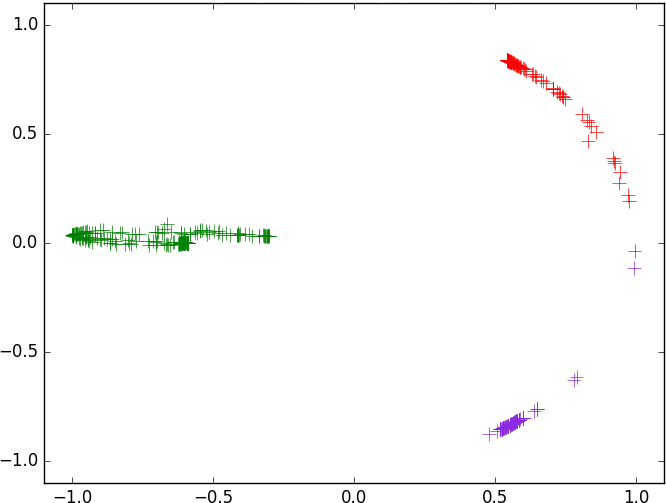} \\}\\
$\scriptstyle \bx_i = \tan(\sigma(\mathbf{W}\bh_i)) $ & \makecell{\\ \includegraphics[height=3.2cm,width=3.2cm]{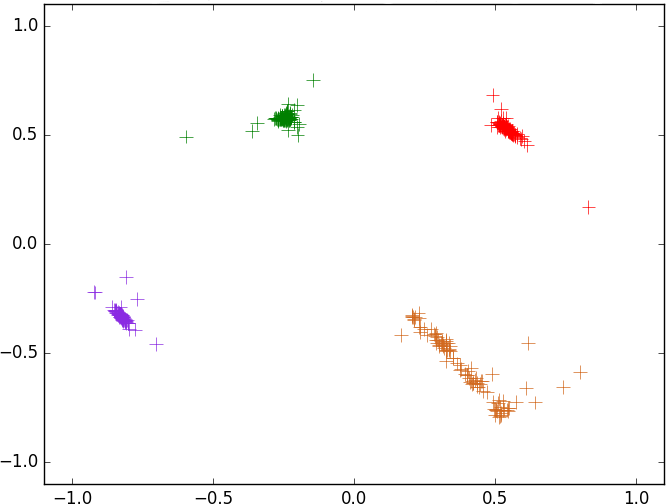} \\ } & \makecell{\\\includegraphics[height=3.2cm,width=3.2cm]{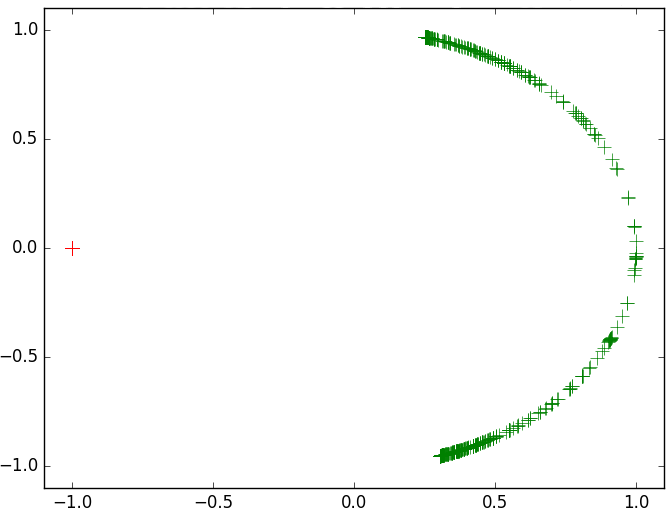} \\} & \makecell{\\ \includegraphics[height=3.2cm,width=3.2cm]{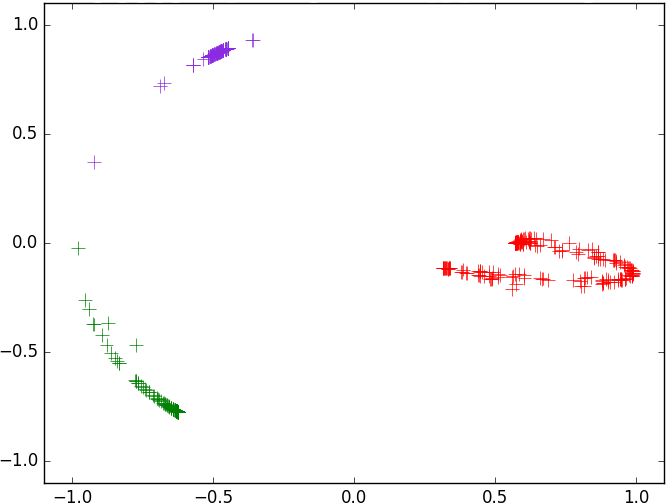} \\}\\
\end{tabular}

\caption{{\bf Embeddings $\bB$ from {\tt SC-EDAE} on {\tt Tetra}, {\tt Chainlink} and {\tt Lsun} high-dimensional datasets:} \footnotesize{
Colors indicate the predicted labels. 
}}
\label{fig:FCPS_synth10_and_11}
\end{figure*}

\vspace{-0.1cm}
\subsection{Complementary experiments on real data}
\label{sec:sup_exp_real}
\subsubsection{Baseline evaluations}
The Table~\ref{tab:ref_ari_nmi} provides complementary results for the baseline evaluations on real datasets. Specifically, it gives the mean Adjusted Rand Index (ARI) and the Normalized Mutual Information (NMI) for \emph{LSC} and \emph{kmeans}$_{++}$. The mean is taken over $10$ replicates on the original datasets, over all epoch and landmark numbers. The results for {\it DAE-LSC} and {\it DAE-kmeans$_{++}$} are averaged over 50 replicates (10 replicates on each of the 5 encodings per DAE structure type), over all epoch and landmark numbers. These results follow the same trend as the accuracy results detailed in main text.

\begin{table*}[!h]
\centering
\scriptsize
\caption{{\bf Mean clustering Adjusted Rand Index (ARI) and Normalized Mutual Information (NMI) for {\tt LSC} and {\tt k-means} on original real datasets and encodings.} {\footnotesize{Evaluations on MNIST, PenDigits, USPS data and their encodings. Bold values highlight the higher results}}}
\setlength\tabcolsep{0.5pt}
\begin{tabular}{|c|c|c|c|c|l|l|l|l|l| }
\cline{2-5}
\cline{7-10}
\multicolumn{1}{c|}{} & \multicolumn{2}{c|}{\small{\bf ARI}} & \multicolumn{2}{c|}{\small{\bf NMI}} & \multicolumn{1}{c|}{} & \multicolumn{2}{c|}{\small{\bf ARI}} & \multicolumn{2}{c|}{\small{\bf NMI}}\\
\hline
\multicolumn{1}{|c|}{\tiny{Data}} & 
\multicolumn{1}{c|}{\tiny{\it LSC}} & 
\multicolumn{1}{c|}{\tiny{\it kmeans$_{++}$}} & \multicolumn{1}{c|}{\tiny{\it LSC}} & 
\multicolumn{1}{c|}{\tiny{\it kmeans$_{++}$}} & \multicolumn{1}{c|}{\tiny{DAE structure}} & \multicolumn{1}{c|}{\tiny{\it DAE-LSC}} & \multicolumn{1}{c|}{\tiny{\it DAE-kmeans$_{++}$}} &
\multicolumn{1}{c|}{\tiny{\it DAE-LSC}} & \multicolumn{1}{c|}{\tiny{\it DAE-kmeans$_{++}$}}\\
\hline
\multirow{6}{*}{\footnotesize{\rotatebox{90}{MNIST}}} & 
\multirow{6}{*}{\rotatebox{90}{{\bf 54.86} $\pm 1.69 $}} & 
\multirow{6}{*}{\rotatebox{90}{39.98 $\pm 0.03$}} &
\multirow{6}{*}{\rotatebox{90}{{\bf 70.54} $\pm 0.83 $}} & \multirow{6}{*}{\rotatebox{90}{52.89 $\pm 0.02$}} &
          500--750--1000 & 78.16 $\pm 10.26 $ & 63.58 $\pm 8.14 $ & 80.88 $\pm 6.58 $ & 70.29 $\pm 5.38 $\\
& & & & & 500--1000--750 & 82.84 $\pm 1.20 $ & 67.66 $\pm 6.36 $ & 84.04 $\pm 1.20 $ & 73.21 $\pm 4.02 $\\
& & & & & 750--500--1000 & 79.20 $\pm 8.42 $ & 65.32 $\pm 7.36 $ & 81.57 $\pm 5.55 $ & 71.50 $\pm 4.64 $\\
& & & & & 750--1000--500 & 82.23 $\pm 6.33 $ & 66.75 $\pm 6.48 $ & 83.52 $\pm 4.13 $ & 72.32 $\pm 4.10 $\\
& & & & & 1000--500--750 & {\bf 83.66} $\pm 4.23 $ & 67.48 $\pm 5.80 $ & {\bf 84.29} $\pm 2.80 $ & 72.80 $\pm 3.52 $\\
& & & & & 1000--750--500 & 83.15 $\pm 4.81$ & 65.28 $\pm 7.48 $ & 84.07 $\pm 2.99$ & 71.69 $\pm 3.09 $\\
\hline
\multirow{6}{*}{\footnotesize{\rotatebox{90}{PenDigits}}} & 
\multirow{6}{*}{\rotatebox{90}{{\bf 68.58} $\pm 3.79 $}} & \multirow{6}{*}{\rotatebox{90}{57.58 $\pm 2.61 $}} &
\multirow{6}{*}{\rotatebox{90}{{\bf 79.78} $\pm 1.42 $}} & \multirow{6}{*}{\rotatebox{90}{69.72 $\pm 0.58 $}} &
          500--750--1000 & {\bf 74.12} $\pm 2.53 $ & 59.62 $\pm 3.79 $ & {\bf 81.06} $\pm 1.43 $ & 69.33 $\pm 2.11 $\\
& & & & & 500--1000--750 & 73.18 $\pm 3.55 $ & 58.97 $\pm 3.41 $ & 80.46 $\pm 1.85 $ & 69.14 $\pm 1.99 $\\
& & & & & 750--500--1000 & 73.47 $\pm 3.12 $ & 58.23 $\pm 3.73 $ & 80.55 $\pm 1.48 $ & 68.56 $\pm 2.25 $\\
& & & & & 750--1000--500 & 73.30 $\pm 3.17 $ & 58.82 $\pm 3.73 $ & 80.38 $\pm 1.62 $ & 68.74 $\pm 2.07 $\\
& & & & & 1000--500--750 & 73.07 $\pm 2.97 $ & 58.92 $\pm 3.35 $ & 80.23 $\pm 1.79 $ & 69.53 $\pm 2.23 $\\
& & & & & 1000--750--500 & 73.40 $\pm 3.17 $ & 58.16 $\pm 3.12 $ & 80.66 $\pm 1.60 $ & 68.83 $\pm 1.90 $\\
\hline
\multirow{6}{*}{\footnotesize{\rotatebox{90}{USPS}}} & 
\multirow{6}{*}{\rotatebox{90}{{\bf 77.09} $\pm 1.52 $}} &
\multirow{6}{*}{\rotatebox{90}{57.70 $\pm 0.12$}} &
\multirow{6}{*}{\rotatebox{90}{{\bf 79.48} $\pm 0.90 $}} &
\multirow{6}{*}{\rotatebox{90}{65.67 $\pm 0.10$}} &
          500--750--1000 & 76.12 $\pm 8.45 $ & 63.62 $\pm 3.02 $ & 80.32 $\pm 4.89 $ & 70.35 $\pm 2.25 $\\
& & & & & 500--1000--750 & 77.34 $\pm 7.71 $ & 64.22 $\pm 3.34 $ & 80.69 $\pm 4.30 $ & 70.37 $\pm 2.16 $\\
& & & & & 750--500--1000 & 73.66 $\pm 6.38 $ & 63.34 $\pm 2.67 $ & 78.77 $\pm 3.81 $ & 70.11 $\pm 1.77 $\\
& & & & & 750--1000--500 & 75.17 $\pm 5.23 $ & 64.87 $\pm 2.66 $ & 80.13 $\pm 3.11 $ & 70.94 $\pm 2.03 $\\
& & & & & 1000--500--750 & 76.15 $\pm 4.29 $ & 64.63 $\pm 2.02 $ & 80.98 $\pm 2.12 $ & 70.80 $\pm 1.36 $\\
& & & & & 1000--750--500 & {\bf 77.75} $\pm 5.02 $ & 63.88 $\pm 2.07 $ & {\bf 81.74} $\pm 2.12 $ & 70.33 $\pm 1.45 $\\
\hline
\end{tabular}
\label{tab:ref_ari_nmi}
\end{table*}

\subsubsection{SC-EDAE ensemble evaluations}

The Table~\ref{tab:ens_ari_nmi} provides complementary results for the ensemble evaluations on real datasets. Specifically, it gives the mean Adjusted Rand Index (ARI) and the Normalized Mutual Information (NMI) for {\tt SC-EDAE}. The mean is taken over $10$ replicates on the encodings. The columns {\it Ens.Init.} and {\it Ens.Ep.} indicate the results for an ensemble approach on the DAE weight initializations ({\it Ens.Init.}, $m=5$) and the DAE training epoch numbers ({\it Ens.Ep.}, $m=5$). The column {\it Ens.Struct.} provides the evaluations for an ensemble approach on various DAE structure types ($m=6$).

\begin{table*}[!h]
\centering
\scriptsize
\caption{{\bf Mean clustering Adjusted Rank Index (ARI) and Normalized Mutual Information (NMI) for the {\tt SC-EDAE} algorithm.} The ensemble is done on initializations, epochs number and structures. Bold values highlight the higher results.}
\setlength\tabcolsep{0.9pt}
\begin{tabular}{|c|l|l|l|l|l|l|l| }
\cline{3-8}
\multicolumn{2}{c|}{\small{}} & \multicolumn{3}{c|}{\small{\bf ARI}} & \multicolumn{3}{c|}{\small{\bf NMI}}\\
\hline
\multicolumn{1}{|c|}{\scriptsize{Data}} & \multicolumn{1}{c|}{\scriptsize{DAE structure}} & \multicolumn{1}{c|}{\scriptsize{\it Ens.Init.}} & \multicolumn{1}{c|}{\scriptsize{\it Ens.Ep.}} & \multicolumn{1}{c|}{\scriptsize{\it Ens.Struct.}} & \multicolumn{1}{c|}{\scriptsize{\it Ens.Init.}} & \multicolumn{1}{c|}{\scriptsize{\it Ens.Ep.}} & \multicolumn{1}{c|}{\scriptsize{\it Ens.Struct.}}\\
\hline
\multirow{6}{*}{\footnotesize{\rotatebox{90}{MNIST}}} & 
   500--750--1000 & 81.87 $\pm 0.49$ & 83.22 $\pm 7.07 $ & \multirow{6}{*}{87.25 $\pm 3.88 $} & 84.69 $\pm 0.28$ & 85.44 $\pm 4.22 $ & \multirow{6}{*}{87.93 $\pm 2.27 $}\\
 & 500--1000--750 & {\bf 90.17} $\pm 0.14$ & 88.84 $\pm 3.93 $ & & {\bf 89.59} $\pm 0.10$ & 88.87 $\pm 2.32 $ &\\
 & 750--500--1000 & 84.66 $\pm 1.71$ & 85.29 $\pm 5.18 $ & & 86.86 $\pm 0.21$ & 86.68 $\pm 3.06 $ &\\
 & 750--1000--500 & 86.18 $\pm 0.20$ & 85.86 $\pm 4.89 $ & & 87.44 $\pm 0.13$ & 87.17 $\pm 2.66 $ &\\
 & 1000--500--750 & 88.47 $\pm 0.27$ & 88.86 $\pm 2.49 $ & & 88.53 $\pm 0.17$ & 88.71 $\pm 1.45 $ &\\
 & 1000--750--500 & 88.59 $\pm 0.38$ & {\bf 90.02} $\pm 1.13$ & & 88.81 $\pm 0.18$ & {\bf 89.44} $\pm 0.81$ &\\
 \hline
\multirow{6}{*}{\footnotesize{\rotatebox{90}{PenDigits}}} & 
   500--750--1000 & {\bf 75.67} $\pm 0.80$ & 76.15 $\pm 1.23 $ & \multirow{6}{*}{74.88 $\pm 1.57$} & {\bf 82.33} $\pm 0.48$ & 82.73 $\pm 0.78 $ & \multirow{6}{*}{81.87 $\pm 0.84$}\\
 & 500--1000--750 & 74.00 $\pm 0.88$ & 74.83 $\pm 1.82 $ & & 80.96 $\pm 0.43$ & 81.68 $\pm 0.99 $ &\\
 & 750--500--1000 & 75.13 $\pm 0.95$ & 75.71 $\pm 0.81$ & & 81.89 $\pm 0.48$ & {\bf 82.19} $\pm 0.51$ &\\
 & 750--1000--500 & 74.98 $\pm 1.20 $ & 75.38 $\pm 1.18 $ & & 81.88 $\pm 0.61 $ & 81.99 $\pm 0.73 $ &\\
 & 1000--500--750 & 75.07 $\pm 0.69$ & {\bf 75.63} $\pm 0.89$ & & 81.86 $\pm 0.39$ & 82.14 $\pm 0.70$ &\\
 & 1000--750--500 & 75.16 $\pm 0.10$ & 75.60 $\pm 0.76$ & & 81.97 $\pm 0.52$ & 82.13 $\pm 0.57$ &\\
\hline
\multirow{6}{*}{\footnotesize{\rotatebox{90}{USPS}}} & 
   500--750--1000 & 75.68 $\pm 1.80 $ & 76.85 $\pm 5.37 $ & \multirow{6}{*}{77.61 $\pm 3.69 $} & 81.93 $\pm 0.81 $ & 82.39 $\pm 3.04 $ & \multirow{6}{*}{83.17 $\pm 1.96 $}\\
 & 500--1000--750 & 76.12 $\pm 0.71$ & 77.67 $\pm 3.63 $ & & 82.29 $\pm 0.39$ & 83.03 $\pm 1.93 $ &\\
 & 750--500--1000 & 74.32 $\pm 1.02 $ & 76.53 $\pm 3.74 $ & & 81.69 $\pm 0.48 $ & 82.07 $\pm 1.87 $ &\\
 & 750--1000--500 & 75.33 $\pm 0.93 $ & 75.70 $\pm 2.82 $ & & 82.34 $\pm 0.43 $ & 82.35 $\pm 1.59 $ &\\
 & 1000--500--750 & 79.93 $\pm 1.64 $ & 77.73 $\pm 3.00 $ & & 84.28 $\pm 0.65 $ & 83.58 $\pm 1.37 $ &\\
 & 1000--750--500 & {\bf 80.96} $\pm 1.99 $ & {\bf 80.79} $\pm 3.21 $ & & {\bf 84.75} $\pm 0.78 $ & {\bf 84.75} $\pm 1.47 $ &\\
 \hline
\end{tabular}
\label{tab:ens_ari_nmi}
\end{table*}

\newpage
\section*{References}

\bibliography{mybibfile}

\begin{thebibliography}{10}
\expandafter\ifx\csname url\endcsname\relax
  \def\url#1{\texttt{#1}}\fi
\expandafter\ifx\csname urlprefix\endcsname\relax\def\urlprefix{URL }\fi
\expandafter\ifx\csname href\endcsname\relax
  \def\href#1#2{#2} \def\path#1{#1}\fi

\bibitem{Yamamoto_14}
M.~Yamamoto, H.~Hwang, A general formulation of cluster analysis with dimension
  reduction and subspace separation, Behaviormetrika 41~(1) (2014) 115--129.

\bibitem{AllabLN17}
K.~Allab, L.~Labiod, M.~Nadif, A {Semi-NMF}-{PCA} unified framework for data
  clustering, {IEEE} Trans. Knowl. Data Eng. 29~(1) (2017) 2--16.

\bibitem{AllabLN18}
K.~Allab, L.~Labiod, M.~Nadif, Simultaneous spectral data embedding and
  clustering, {IEEE} Trans. Neural Netw. Learning Syst. 29~(12) (2018)
  6396--6401.

\bibitem{Hinton06}
G.~E. Hinton, R.~{Salakhutdinov}, {Reducing the Dimensionality of Data with
  Neural Networks}, Science 313 (2006) 504--507.

\bibitem{bengio2009learning}
Y.~Bengio, et~al., Learning deep architectures for ai, Foundations and
  trends{\textregistered} in Machine Learning 2~(1) (2009) 1--127.

\bibitem{Baldi12}
P.~Baldi, Autoencoders, unsupervised learning, and deep architectures, in:
  Unsupervised and Transfer Learning - Workshop held at {ICML} 2011, 2012, pp.
  37--50.

\bibitem{BengioYAV13}
Y.~Bengio, L.~Yao, G.~Alain, P.~Vincent, Generalized denoising auto-encoders as
  generative models, in: NIPS 2013, 2013, pp. 899--907.

\bibitem{bengio2007greedy}
Y.~Bengio, P.~Lamblin, D.~Popovici, H.~Larochelle, Greedy layer-wise training
  of deep networks, in: Advances in neural information processing systems,
  2007, pp. 153--160.

\bibitem{ShaoLDF15}
M.~Shao, S.~Li, Z.~Ding, Y.~Fu, Deep linear coding for fast graph clustering,
  in: {IJCAI} 2015, 2015, pp. 3798--3804.

\bibitem{TianGCCL14}
F.~Tian, B.~Gao, Q.~Cui, E.~Chen, T.~Liu, Learning deep representations for
  graph clustering, in: {AAAI} 2014, 2014, pp. 1293--1299.

\bibitem{WangHWW14}
W.~Wang, Y.~Huang, Y.~Wang, L.~Wang, Generalized autoencoder: {A} neural
  network framework for dimensionality reduction, in: {IEEE} {CVPR} Workshops
  2014, 2014, pp. 496--503.

\bibitem{HuangHWW14}
P.~Huang, Y.~Huang, W.~Wang, L.~Wang, Deep embedding network for clustering,
  in: {ICPR} 2014, 2014, pp. 1532--1537.

\bibitem{Leyli-AbadiLN17}
M.~Leyli{-}Abadi, L.~Labiod, M.~Nadif, Denoising autoencoder as an effective
  dimensionality reduction and clustering of text data, in: {PAKDD} 2017, 2017,
  pp. 801--813.

\bibitem{YangCHWWZ16}
L.~Yang, X.~Cao, D.~He, C.~Wang, X.~Wang, W.~Zhang, Modularity based community
  detection with deep learning, in: {IJCAI} 2016, 2016, pp. 2252--2258.

\bibitem{BanijamaliG17}
E.~Banijamali, A.~Ghodsi, Fast spectral clustering using autoencoders and
  landmarks, in: {ICIAR} 2017, 2017, pp. 380--388.

\bibitem{WangDF17}
S.~Wang, Z.~Ding, Y.~Fu, Feature selection guided auto-encoder, in: {AAAI}
  2017, 2017, pp. 2725--2731.

\bibitem{XieGF16}
J.~Xie, R.~B. Girshick, A.~Farhadi, Unsupervised deep embedding for clustering
  analysis, in: {ICML}, 2016, pp. 478--487.

\bibitem{YangFSH17}
B.~Yang, X.~Fu, N.~D. Sidiropoulos, M.~Hong, Towards k-means-friendly spaces:
  Simultaneous deep learning and clustering, in: Proceedings of the 34th
  International Conference on Machine Learning, {ICML} 2017, Sydney, NSW,
  Australia, 6-11 August 2017, 2017, pp. 3861--3870.

\bibitem{Tian2017}
K.~Tian, S.~Zhou, J.~Guan, Deepcluster: A general clustering framework based on
  deep learning, in: M.~Ceci, J.~Hollm{\'e}n, L.~Todorovski, C.~Vens,
  S.~D{\v{z}}eroski (Eds.), Machine Learning and Knowledge Discovery in
  Databases, 2017.

\bibitem{Yang2016}
L.~Yang, X.~Cao, D.~He, C.~Wang, X.~Wang, W.~Zhang, Modularity based community
  detection with deep learning, in: Proceedings of the Twenty-Fifth
  International Joint Conference on Artificial Intelligence, IJCAI'16, 2016.

\bibitem{SeuretALI17}
M.~Seuret, M.~Alberti, M.~Liwicki, R.~Ingold, Pca-initialized deep neural
  networks applied to document image analysis, in: 14th {IAPR} International
  Conference on Document Analysis and Recognition, {ICDAR} 2017, Kyoto, Japan,
  November 9-15, 2017, 2017, pp. 877--882.

\bibitem{erhan2010does}
D.~Erhan, Y.~Bengio, A.~Courville, P.-A. Manzagol, P.~Vincent, S.~Bengio, Why
  does unsupervised pre-training help deep learning?, Journal of Machine
  Learning Research 11~(Feb) (2010) 625--660.

\bibitem{guo2017improved}
X.~Guo, L.~Gao, X.~Liu, J.~Yin, Improved deep embedded clustering with local
  structure preservation, in: International Joint Conference on Artificial
  Intelligence (IJCAI-17), 2017, pp. 1753--1759.

\bibitem{xie2016unsupervised}
J.~Xie, R.~Girshick, A.~Farhadi, Unsupervised deep embedding for clustering
  analysis, in: International conference on machine learning, 2016, pp.
  478--487.

\bibitem{Ji2017}
P.~Ji, T.~Zhang, H.~Li, M.~Salzmann, I.~Reid, Deep subspace clustering
  networks, in: I.~Guyon, U.~V. Luxburg, S.~Bengio, H.~Wallach, R.~Fergus,
  S.~Vishwanathan, R.~Garnett (Eds.), Advances in Neural Information Processing
  Systems 30, Curran Associates, Inc., 2017, pp. 24--33.

\bibitem{Liu:2010}
W.~Liu, J.~He, S.-F. Chang, Large graph construction for scalable
  semi-supervised learning, in: Proceedings of the 27th International
  Conference on International Conference on Machine Learning, ICML'10, 2010.

\bibitem{Chen11LandmarkSpectral}
X.~Chen, D.~Cai, Large scale spectral clustering with landmark-based
  representation, in: Twenty-Fifth Conference on Artificial Intelligence
  (AAAI'11), 2011.

\bibitem{Guo:2017}
X.~Guo, L.~Gao, X.~Liu, J.~Yin, Improved deep embedded clustering with local
  structure preservation, in: Proceedings of the 26th International Joint
  Conference on Artificial Intelligence, IJCAI'17, 2017.

\bibitem{verma2003comparison}
D.~Verma, M.~Meila, A comparison of spectral clustering algorithms, University
  of Washington Tech Rep UWCSE030501 1 (2003) 1--18.

\bibitem{von2007tutorial}
U.~Von~Luxburg, A tutorial on spectral clustering, Statistics and computing
  17~(4) (2007) 395--416.

\bibitem{shi2000normalized}
J.~Shi, J.~Malik, Normalized cuts and image segmentation, IEEE Transactions on
  pattern analysis and machine intelligence 22~(8) (2000) 888--905.

\bibitem{ng2002spectral}
A.~Y. Ng, M.~I. Jordan, Y.~Weiss, On spectral clustering: Analysis and an
  algorithm, in: Advances in neural information processing systems, 2002, pp.
  849--856.

\bibitem{meila2001learning}
M.~Meila, J.~Shi, Learning segmentation by random walks, in: Advances in neural
  information processing systems, 2001, pp. 873--879.

\bibitem{chen2011large}
X.~Chen, D.~Cai, Large scale spectral clustering with landmark-based
  representation., in: AAAI, Vol.~5, 2011, p.~14.

\bibitem{hinton1994autoencoders}
G.~E. Hinton, R.~S. Zemel, Autoencoders, minimum description length and
  helmholtz free energy, in: Advances in neural information processing systems,
  1994, pp. 3--10.

\bibitem{Strehl:2003}
A.~Strehl, J.~Ghosh, Cluster ensembles --- a knowledge reuse framework for
  combining multiple partitions, J. Mach. Learn. Res. 3 (2003) 583--617.

\bibitem{vega2011survey}
S.~Vega-Pons, J.~Ruiz-Shulcloper, A survey of clustering ensemble algorithms,
  International Journal of Pattern Recognition and Artificial Intelligence
  25~(03) (2011) 337--372.

\bibitem{Maila-2001}
M.~Maila, J.~Shi, A random walks view of spectral segmentation, in: AI and
  STATISTICS (AISTATS), 2001.

\bibitem{glorot2010understanding}
X.~Glorot, Y.~Bengio, Understanding the difficulty of training deep feedforward
  neural networks, in: Proceedings of the thirteenth international conference
  on artificial intelligence and statistics, 2010, pp. 249--256.

\bibitem{reddi2018convergence}
S.~J. Reddi, S.~Kale, S.~Kumar, On the convergence of adam and beyond, in:
  International Conference on Learning Representations, 2018.

\bibitem{lecun1998gradient}
Y.~LeCun, L.~Bottou, Y.~Bengio, P.~Haffner, Gradient-based learning applied to
  document recognition, Proceedings of the IEEE 86~(11) (1998) 2278--2324.

\bibitem{alimoglu1996methods}
F.~Alimoglu, E.~Alpaydin, Methods of combining multiple classifiers based on
  different representations for pen-based handwritten digit recognition, in:
  Proceedings of the Fifth Turkish Artificial Intelligence and Artificial
  Neural Networks Symposium, Citeseer, 1996.

\bibitem{vapnik1998statistical}
V.~Vapnik, Statistical learning theory. 1998, Wiley, New York, 1998.

\bibitem{arthur2007k}
D.~Arthur, S.~Vassilvitskii, k-means++: The advantages of careful seeding, in:
  Proceedings of the eighteenth annual ACM-SIAM symposium on Discrete
  algorithms, Society for Industrial and Applied Mathematics, 2007, pp.
  1027--1035.

\bibitem{steinley2004properties}
D.~Steinley, Properties of the hubert-arable adjusted rand index.,
  Psychological methods 9~(3) (2004) 386.

\bibitem{fard2018deep}
M.~M. Fard, T.~Thonet, E.~Gaussier, Deep $ k $-means: Jointly clustering with $
  k $-means and learning representations, arXiv preprint arXiv:1806.10069.

\bibitem{maaten2008visualizing}
L.~v.~d. Maaten, G.~Hinton, Visualizing data using t-sne, Journal of machine
  learning research 9~(Nov) (2008) 2579--2605.

\end{thebibliography}

\end{document}